\documentclass{article}

\PassOptionsToPackage{numbers,compress}{natbib}
\usepackage[final]{neurips_2020}
\usepackage[utf8]{inputenc}
\usepackage[T1]{fontenc}
\usepackage{hyperref}
\usepackage{url}
\usepackage{booktabs}
\usepackage{colonequals}
\usepackage{amsfonts}
\usepackage{calc}
\usepackage{amssymb}
\usepackage{nicefrac}
\usepackage{tabularx}
\usepackage{microtype}
\usepackage{mathtools}
\usepackage{graphicx}
\usepackage{multirow}
\usepackage{xfrac}
\usepackage{tikz}
\usepackage{algpseudocode}
\usepackage{wrapfig}
\usepackage[tight,footnotesize]{subfigure}
\usepackage{amsmath}
\usepackage{enumerate} 
\usepackage[font=small,labelfont=bf]{caption}
\usepackage{xcolor}
\usepackage{amsthm}
\usepackage{enumitem}
\usepackage[noend, ruled, vlined]{algorithm2e}

\usepackage{amsmath,amsfonts,bm}

\def\<{ {\langle} }
\def\>{ {\rangle} }

\def\figref#1{figure~\ref{#1}}

\def\eqref#1{equation~\ref{#1}}

\def\1{\bm{1}}

\newcommand{\norm}[1]{\left\lVert#1\right\rVert}

\DeclareMathAlphabet{\mathsfit}{\encodingdefault}{\sfdefault}{m}{sl}
\SetMathAlphabet{\mathsfit}{bold}{\encodingdefault}{\sfdefault}{bx}{n}

\def\mbi#1{\boldsymbol{#1}} 
\def\mbf#1{\mathbf{#1}}

\def\sR{{\mathbb{R}}}

\DeclareMathOperator*{\argmin}{arg\,min}

\ifdefined\nonewproofenvironments\else
\ifdefined\ispres\else
\newtheorem{theorem}{Theorem}
\newtheorem{lemma}[theorem]{Lemma}

\renewenvironment{proof}{\noindent\textbf{Proof.}\hspace*{.3em}}{\qed\\}
\newenvironment{proof-sketch}{\noindent\textbf{Proof Sketch}
  \hspace*{1em}}{\qed\bigskip\\}
\newenvironment{proof-idea}{\noindent\textbf{Proof Idea}
  \hspace*{1em}}{\qed\bigskip\\}
\newenvironment{proof-of-lemma}[1][{}]{\noindent\textbf{Proof of Lemma {#1}}
  \hspace*{1em}}{\qed\\}
\newenvironment{proof-of-theorem}[1][{}]{\noindent\textbf{Proof of Theorem {#1}}
  \hspace*{1em}}{\qed\\}
\newenvironment{proof-attempt}{\noindent\textbf{Proof Attempt}
  \hspace*{1em}}{\qed\bigskip\\}

\fi
\newtheorem{observation}[theorem]{Observation}

\fi
\makeatletter
\@addtoreset{equation}{section}
\makeatother

\def\mathunderline#1#2{\color{#1}\underbracket{{\color{black}#2}}\color{black}}
\def\mathunderline#1#2{\color{#1}\underbracket[1pt][1pt]{{\color{black}#2}}\color{black}}

\newcommand{\hyper}{\mbi{\lambda}}
\newcommand{\hw}{\mbi{\theta}}
\newcommand{\stnhw}{\mbi{\phi}}
\newcommand{\dstnhw}{\mbi{\theta}}

\newcommand{\w}{\mbf{w}}

\newcommand{\pert}{\mbi{\epsilon}}
\newcommand{\scal}{\mbi{\sigma}}

\newcommand{\hwr}{\mbi{\Theta}}
\newcommand{\stnbr}{\mbi{\Phi}}
\newcommand{\trainobj}{\mathcal{L}_T}
\newcommand{\valobj}{\mathcal{L}_V}

\newcommand{\ba}{\mbf{a}}
\newcommand{\bx}{\mbf{x}}
\newcommand{\by}{\mbf{y}}
\newcommand{\scalm}{\mbi{\Sigma}}
\newcommand{\stnhwVec}{{\boldsymbol{\phi}}}
\newcommand{\jac}{\mathbf{J}}
\newcommand{\hess}{\mathbf{H}}
\newcommand{\gnHess}{\mathbf{G}}
\newcommand{\res}{\mathbf{r}}
\newcommand*{\vtilde}[2][0pt]{
  \setbox0=\hbox{$#2$}%
  \widetilde{\mathrlap{\phantom{\rule{\wd0}{\ht0+{#1}}}}\smash{#2}}%
}

\newcommand*\circled[1]{\tikz[baseline=(char.base)]{
            \node[shape=circle,draw,inner sep=2pt] (char) {#1};}}

\title{$\Delta$-STN: Efficient Bilevel Optimization for Neural Networks using Structured Response Jacobians}

\author{
Juhan Bae\\
University of Toronto\\
Vector Institute\\
\texttt{jbae@cs.toronto.edu} \\
\And
Roger Grosse\\
University of Toronto\\
Vector Institute\\
\texttt{rgrosse@cs.toronto.edu}
}

\begin{document}

\maketitle

\begin{abstract}
Hyperparameter optimization of neural networks can be elegantly formulated as a bilevel optimization problem. While research on bilevel optimization of neural networks has been dominated by implicit differentiation and unrolling, hypernetworks such as Self-Tuning Networks (STNs) have recently gained traction due to their ability to amortize the optimization of the inner objective. In this paper, we diagnose several subtle pathologies in the training of STNs. Based on these observations, we propose the $\Delta$-STN, an improved hypernetwork architecture which stabilizes training and optimizes hyperparameters much more efficiently than STNs. The key idea is to focus on accurately approximating the best-response Jacobian rather than the full best-response function; we achieve this by reparameterizing the hypernetwork and linearizing the network around the current parameters. We demonstrate empirically that our $\Delta$-STN can tune regularization hyperparameters (e.g.~weight decay, dropout, number of cutout holes) with higher accuracy, faster convergence, and improved stability compared to existing approaches.
\end{abstract}

\section{Introduction}
\label{sec:introduction}
Tuning regularization hyperparameters such as weight decay, dropout~\citep{srivastava2014dropout}, and data augmentation is indispensable for state-of-the-art performance in a challenging dataset such as ImageNet~\citep{imagenet_cvpr09,smirnov2014comparison,cubuk2019randaugment}. An automatic approach to adapting these hyperparameters would improve performance and simplify the engineering process. Although black box methods for tuning hyperparameters such as grid search, random search~\citep{bergstra2012random}, and Bayesian optimization~\citep{snoek2012practical,snoek2015scalable} work well in low-dimensional hyperparameter spaces, they are computationally expensive, require many runs of training, and require that hyperparameter values be fixed throughout training.

Hyperparameter optimization can be elegantly formulated as a bilevel optimization problem~\citep{colson2007overview,franceschi2018bilevel}. Let $\w \in \sR^m$ denote parameters (e.g.~weights and biases) and $\hyper \in \sR^h$ denote hyperparameters (e.g.~weight decay). Let $\valobj$ and $\trainobj$ denote validation and training objectives, respectively. We aim to find the optimal hyperparameters $\hyper^*$ that minimize the validation objective at the end of training. Mathematically, the bilevel objective can be formulated as follows\footnote{The uniqueness of $\argmin$ is assumed throughout this paper.}: 
\begin{equation}
    \hyper^* = \argmin_{\hyper \in \sR^h} \valobj (\hyper, \w^*)
    \text{~subject to~} \w^* = \argmin_{\w \in \sR^m} \trainobj (\hyper, \w)
    \label{eq:bilvel-hyperparameter-obj}
\end{equation}
In machine learning, most work on bilevel optimization has focused on implicit differentiation~\citep{larsen1996design,pedregosa2016hyperparameter} and unrolling~\citep{maclaurin2015gradient,franceschi2017forward}. A more recent approach, which we build on in this work, explicitly approximates the \emph{best-response (rational reaction) function} $\res(\hyper) = \argmin_{\w} \mathcal{L}_T (\hyper, \w)$ with a hypernetwork~\citep{schmidhuber1992learning,ha2016hypernetworks} and jointly optimizes the hyperparameters and the hypernetwork~\citep{lorraine2018stochastic}. The hypernetwork approach is advantageous because training the hypernetwork amortizes the inner-loop optimization work required by both implicit differentiation and unrolling. Since best-response functions are challenging to represent due to their high dimensionality, Self-Tuning Networks (STNs)~\citep{mackay2019self} construct a structured hypernetwork to each layer of the neural network, thereby allowing the efficient and scalable approximation of the best-response function.

In this work, we introduce the $\Delta$-STN, a novel architecture for bilevel optimization that fixes several subtle pathologies in training STNs. We first improve the conditioning of the Gauss-Newton Hessian and fix undesirable bilevel optimization dynamics by reparameterizing the hypernetwork, thereby enhancing the stability and convergence in training. Based on the proposed parameterization, we further introduce a modified training scheme that reduces variance in parameter updates and eliminates any bias induced by perturbing the hypernetwork.

Next, we linearize the best-response hypernetwork to yield an affine approximation of the best-response function. In particular, we linearize the dependency between the network's parameters and predictions so that the training algorithm is encouraged to accurately approximate the Jacobian of the best-response function. Empirically, we evaluate the performance of $\Delta$-STNs on linear models, image classification tasks, and language modelling tasks, and show our method consistently outperform the baselines, achieving better generalization performance in less time.

\section{Background}
\label{sec:background}

\subsection{Bilevel Optimization with Gradient Descent}
\label{sec:background-bilevel}
A bilevel problem (see~\citet{colson2007overview} for an overview) consists of two sub-problems, where one problem is nested within another. The \emph{outer-level problem (leader)} must be solved subject to the optimal value of the \emph{inner-level problem (follower)}. A general formulation of the bilevel problem is as follows:
\begin{align}
    \min_{\mbf{x} \in \sR^h} f(\mbf{x}, \mbf{y}^*) \text{~subject to~} \mbf{y}^* = \argmin_{\mbf{y} \in \sR^m} g(\mbf{x}, \mbf{y}),
    \label{eq:bilevel-problem}
\end{align}
where $f, g \colon \sR^h \times \sR^m \to \sR$ denote outer- and inner-level objectives (e.g.~$\valobj$ and $\trainobj$), and $\mbf{x} \in \sR^h$ and $\mbf{y} \in \sR^m$ denote outer- and inner-level variables (e.g.~$\hyper$ and $\w$). Many problems can be cast as bilevel objectives in machine learning, including hyperparameter optimization, generative adversarial networks (GANs)~\citep{goodfellow2014generative,jin2019local}, meta-learning~\citep{franceschi2018bilevel}, and neural architecture search~\citep{zoph2016neural,liu2018darts,cortes2017adanet}.

A na{\"\i}ve application of simultaneous gradient descent on training and validation objectives will fail due to the hierarchy induced by the bilevel structure~\citep{fiez2019convergence,wang2019solving}. A more principled approach in solving the bilevel problem is to incorporate the best-response function. Substituting the best-response function in the validation objective converts the bilevel problem to a single-level problem: 
\begin{align}
    \min_{\hyper \in \sR^h} \valobj (\hyper, \res(\hyper))
\end{align}
We refer readers to~\citet{fiez2019convergence} on a more detailed analysis of the best-response (rational reaction) function. The Jacobian of the best-response function is as follows:
\begin{align}
    \frac{\partial \res}{\partial \hyper} (\hyper) = -  \left(\frac{\partial^2 \trainobj}{\partial \w^2} (\hyper, \res(\hyper)) \right)^{-1} \frac{\partial^2 \trainobj}{ \partial \w \partial \hyper} (\hyper, \res(\hyper))
\end{align}
The gradient of the validation objective is composed of \emph{direct} and \emph{response gradients}. While the direct gradient captures the direct dependence on the hyperparameters in the validation objective, the response gradient captures how the optimal parameter responds to the change in the hyperparameters: 
\begin{align}
    \frac{\mathrm{d} \valobj}{\mathrm{d} \hyper} (\hyper, \res(\hyper)) &= \underbrace{\frac{\partial \valobj}{\partial \hyper}(\hyper, \w)}_{\text{Direct Gradient}}  + \underbrace{\mathunderline{red}{\left(\frac{\partial \res}{\partial \hyper} (\hyper) \right)^{\top}}\frac{\partial \valobj}{\partial \w} (\hyper, \res(\hyper)) }_{\text{Response Gradient}} 
    \label{eq:valid-loss-decompose}
\end{align}
For most hyperparameter optimization problems, the regularization penalty is not imposed in the validation objective, so the direct gradient is $\mbf{0}$. Therefore, either computing or approximating the best-response Jacobian (\textcolor{red}{\textbf{--}}) is essential for computing the gradient of the outer objective.

\subsection{Best-Response Hypernetwork}
\label{sec:background-hypernetwork}
\begin{figure}[t]
\vspace{-5mm}
\small
\begin{center}
\includegraphics[width=0.8\textwidth]{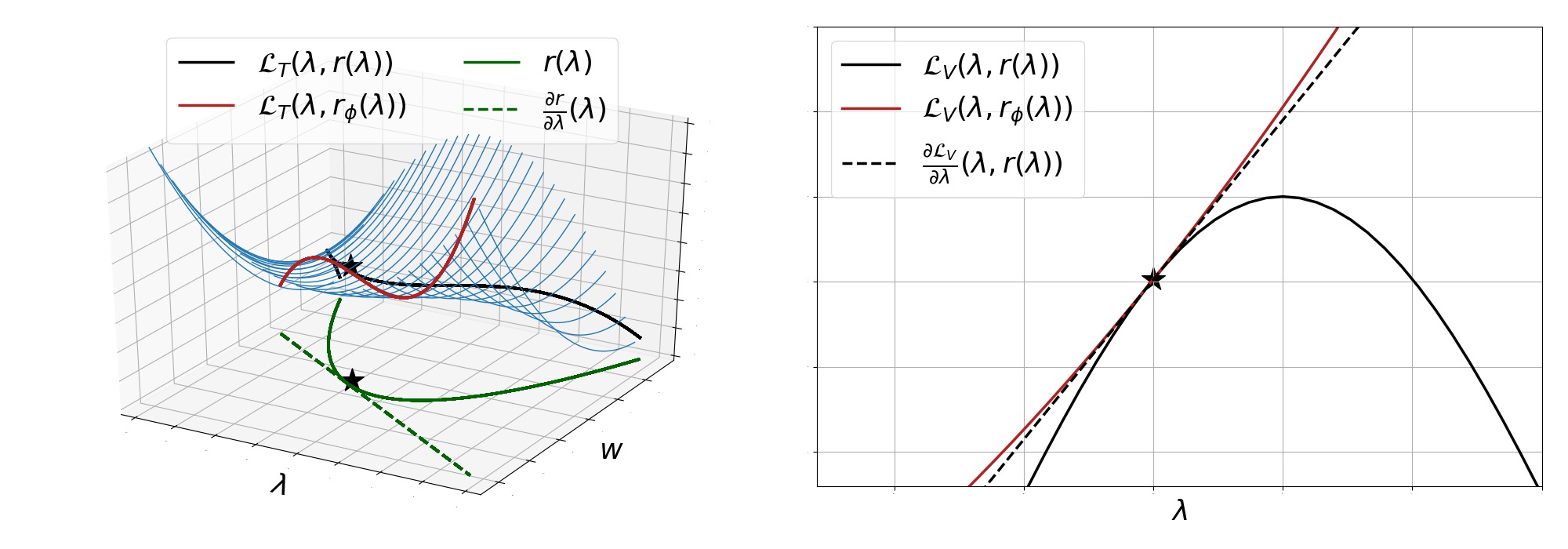}
\end{center}
\caption{\textbf{(Left)} The loss surface of the training objective. The best-response hypernetwork $\res_{\boldsymbol{\phi}}$ locally approximates the best-response function at the current configuration ($\star$). \textbf{(Right)} The loss surface of the validation objective. The local best-response hypernetwork embedded in the validation objective allows to capture the response gradient in Eqn.~\ref{eq:valid-loss-decompose}.}
\label{fig:bilevel-toy}
\vspace{-4mm}
\end{figure}

\citet{lorraine2018stochastic} locally model the best-response function with a hypernetwork. Let $\stnhw \in \sR^p$ be the best-response (hypernetwork) parameters and $p(\pert | \scal)$ be a zero-mean diagonal Gaussian distribution with a fixed perturbation scale $\scal \in \sR^h_+$. The objective for the hypernetwork is as follows:
\begin{align}
    \min_{\stnhw \in \sR^p} \mathbb{E}_{\pert \sim p(\pert|\scal)} \left[\trainobj (\hyper + \pert, \res_{\stnhw} (\hyper + \pert)) \right],
    \label{eq:lor-objective}
\end{align}
where $\res_{\stnhw}$ is the best-response hypernetwork. Intuitively, the hyperparameters are perturbed so that the hypernetwork can locally learn the best-response curve, as shown in \figref{fig:bilevel-toy}. Self-Tuning Networks (STNs)~\citep{mackay2019self} extend this approach by constructing a structured hypernetwork for each layer in neural networks in order to scale well with the number of hyperparameters. They parameterize the best-response function as:
\begin{align}
    \res_{\stnhw} (\hyper) = \stnbr \hyper + \stnhw_0,
    \label{eq:old-param}
\end{align}
where $\stnhw_0 \in \sR^m$ and $\mbi{\Phi} \in \sR^{m \times h}$, and further impose a structure on weights of the neural network, as detailed in section~\ref{section:nn}. The STN performs alternating gradient descents on hypernetwork, hyperparameters and perturbation scale, where the validation objective is formulated as:
\begin{align}
    \min_{\hyper \in \sR^h, \scal \in \sR^h_+}\mathbb{E}_{\pert \sim p(\pert| \scal)} \left[\mathcal{L}_V (\hyper + \pert, \res_{\boldsymbol{\phi}}(\hyper + \pert))\right] - \tau \mathbb{H} \left[p(\pert | \scal) \right]
    \label{eq:stn-val-obj}
\end{align}
Here, $\mathbb{H}[\cdot]$ is an entropy term weighted by $\tau \in \sR_{+}$ to prevent the perturbation scale from getting too small. Note that the above objective is similar to that of variational inference, where the first term is analogous to the negative log-likelihood. When $\tau$ ranges from 0 to 1, the objective interpolates between variational inference and variational optimization~\citep{staines2012variational}. The full algorithm for the STN is described in Alg.~\ref{alg:stn} (appendix~\ref{appendix:algo-stn}).

\section{Method}
\label{sec:method}
We now describe our main contributions: a centered parameterization of the hypernetwork, a modified training objective, and a linearization of the best-response hypernetwork. For simplicity, we first present these contributions in the context of a full linear hypernetwork (which is generally impractical to represent) and then explain how our contributions can be combined with the compact hypernetwork structure used in STNs.

\subsection{Centered Parameterization of the Best-Response Hypernetwork}
\label{sec:method-centered-param}
Our first contribution is to use a centered parameterization of the hypernetwork. In particular, observe that if the hyperparameters are transformed as $\vtilde[-0.5pt]{\hyper} = \hyper + \ba$, for any vector $\ba \in \sR^h$, then the hypernetwork parameters can be transformed as $\vtilde[-0.5pt]{\stnhw}_0 = \stnbr \ba + \stnhw_0$ to represent the same mapping in Eqn.~\ref{eq:old-param}. Therefore, we can reparameterize the model by adding an offset to the hyperparameters if it is advantageous for optimization. As it is widely considered beneficial to center inputs and activations at 0~\citep{lecun1998efficient,ioffe2015batch,montavon2012deep}, we propose to adopt the following centered parameterization of the hypernetwork:
\begin{align}
    \res_{\dstnhw}(\hyper, \hyper_0) = \hwr (\hyper - \hyper_0) + \w_0,
    \label{eq:new-param}
\end{align}
where $\w_0 \in \sR^m$ and $\mbi{\Theta} \in \sR^{m \times h}$. Intuitively, if $\hyper_0$ is regarded as the ``current hyperparameters'', then $\w_0$ can be seen as the ``current weights'', and $\hwr$ determines how the weights are adjusted in response to a perturbation to $\hyper$. We provide two justifications why the centered parameterization is advantageous; the first justification relates to the conditioning of the single-level optimization problem for $\stnhwVec$, while the second involves a phenomenon particular to bilevel optimization. 

For the first justification, consider optimizing the hypernetwork parameters $\stnhwVec$ for a fixed $\bar{\hyper} \in \sR^{h+1}$, where $\bar{\hyper}$ is defined as a vector formed by appending additional homogeneous coordinate with value 1 to incorporate the offset term. The speed of convergence of gradient descent is closely related to the condition number of the Hessian $\hess_\stnhwVec = \nabla^2_\stnhwVec \mathcal{L}_T$~\citep{nocedal2006numerical}. For neural networks, it is common to approximate $\hess_\stnhwVec$ with the Gauss-Newton Hessian~\citep{lecun1998efficient,martens2014new}, which linearizes the network around $\w$:
\begin{equation}
    \hess_\w \approx \gnHess_\w \triangleq \mathbb{E} \left[\jac_{\by\w}^{\top} \hess_\by \jac_{\by\w} \right],
\end{equation}
where $\jac_{\by\w} = \sfrac{\partial \by}{\partial \w}$ is the weight-output Jacobian, $\hess_\by = \nabla^2_\by \mathcal{L}_T$ is the Hessian of the loss with respect to the network outputs $\by$, and the expectation is with respect to the training distribution.

\begin{observation}
The Gauss-Newton Hessian with respect to the hypernetwork is given by:
\begin{equation}
    \gnHess_\stnhwVec = \mathbb{E} \left[\hat{\hyper} \hat{\hyper}^{\top} \otimes \jac_{\by\w}^{\top} \hess_\by \jac_{\by\w} \right],
\end{equation}
where $\hat{\hyper} = \bar{\hyper} + \bar{\pert}$ are the sampled hyperparameters and $\bar{\pert}$ is the perturbation vector appended with additional homogeneous coordinate with value 0.
\label{obs:gnhhw}
\end{observation}
See appendix~\ref{obs:gnhhw-proof} for the derivation. Heuristically, we can approximate $\gnHess_\stnhwVec$ by pushing the expectation inside the Kronecker product, a trick that underlies the K-FAC optimizer~\citep{martens2015optimizing,zhang2017noisy}:
\begin{equation}
    \gnHess_\stnhwVec \approx \mathbb{E} [\hat{\hyper} \hat{\hyper}^{\top} ] \otimes \mathbb{E} \left[\jac_{\by\w}^{\top} \hess_\by \jac_{\by\w} \right] = \mathbb{E} [\hat{\hyper} \hat{\hyper}^{\top}] \otimes \gnHess_\w, 
    \label{eqn:push_expectation}
\end{equation}
where $\gnHess_\w$ is the Gauss-Newton Hessian for the network itself. We now note the following well-known fact about the condition number of the Kronecker product (see appendix~\ref{obs:condition-kfac-proof} for the proof):
\begin{observation}
Let $\kappa(\mbf{A})$ denote the condition number of a square positive definite matrix $\mbf{A}$. Given square positive definite matrices $\mbf{A}$ and $\mbf{B}$, the condition number of $\mbf{A} \otimes \mbf{B}$ is given by $\kappa ( \mbf{A} \otimes \mbf{B} )=\kappa(\mbf{A}) \kappa(\mbf{B})$. 
\label{obs:condition-kfac}
\end{observation}
Hence, applying Observation~\ref{obs:condition-kfac} to Eqn.~\ref{eqn:push_expectation}, we would like to make the factor $\mathbb{E}[\hat{\hyper} \hat{\hyper}^{\top}]$ well-conditioned in order to make the overall optimization problem well-conditioned. We apply the well-known decomposition of the second moments:
\begin{align}
    \mathbb{E} [\hat{\hyper} \hat{\hyper}^{\top} ] = \text{Cov} (\hat{\hyper}) + \mathbb{E} [\hat{\hyper}] \mathbb{E}[\hat{\hyper}]^{\top} = \text{Cov}(\pert) + \bar{\hyper} \bar{\hyper}^\top
    \label{eq:decomp}
\end{align}
In the context of our method, the first term is the diagonal matrix $\text{diag}(\boldsymbol{\sigma}^2)$, whose entries are typically small. The second term is a rank-1 matrix whose nonzero eigenvalue is $\|\hyper\|^2$. If $\hyper$ is high-dimensional or far from $\mbf{0}$, the latter will dominate, and the problem will be ill-conditioned due to the large outlier eigenvalue. Therefore, to keep the problem well-conditioned, we would like to ensure that $\|\hyper\|^2$ is small and one way to do this is to use a centered parameterization that sends $\hyper \mapsto \mathbf{0}$, such as Eqn.~\ref{eq:new-param}. This suggests that centering should improve the conditioning of the inner objective.

The above analysis justifies why we would expect centering to speed up training of the hypernetwork, taken as a single-level optimization problem. However, there is also a second undesirable effect of uncentered representations involving \emph{interactions} between the inner and outer optimizations. The result of a single batch gradient descent update to $\stnbr$ is as follows:
\begin{align}
    \stnbr^{(1)} &= \stnbr^{(0)} - \alpha \nabla_{\stnbr} (\trainobj(\hyper^{(0)}, \res_{\stnhw^{(0)}}(\hyper^{(0)})) )= \stnbr^{(0)} - \alpha (\nabla_\w \trainobj(\hyper^{(0)}, \w^{(0)})) (\hyper^{(0)})^{\top}, 
    \label{eqn:uncentered_dynamics}
\end{align}
where $\hyper^{(0)}$ denotes the hyperparameters sampled in the first iteration and $\alpha$ denotes the learning rate of the inner objective. This results in a failure of credit assignment: it induces the (likely mistaken) belief that adjusting $\hyper$ in the direction of $\hyper$ will move the optimal weights in the direction of $-\nabla_\w \trainobj$. Plugging in Eqn.~\ref{eqn:uncentered_dynamics}, this leads to the following hyperparameter gradient $\nabla_{\hyper} (\valobj)$:
\begin{align}
    \nabla_{\hyper} (\valobj(\hyper, \res_{\stnhwVec^{(1)}}(\hyper))) &= \nabla_{\hyper} \valobj(\hyper, \w^{(1)}) + (\stnbr^{(1)})^{\top} \nabla_\w \valobj(\hyper, \w^{(1)}) \\
    &= \nabla_{\hyper} \valobj(\hyper, \w^{(1)}) + (\stnbr^{(0)})^{\top} \nabla_\w \valobj(\hyper, \w^{(1)}) \notag \\
    &\phantom{=} - \alpha (\nabla_\w \mathcal{L}_T(\hyper^{(0)}, \w^{(0)}))^{\top} (\nabla_\w \valobj(\hyper, \w^{(1)})) \hyper^{(0)}
    \label{eq:valid-undesirable}
\end{align}
In Eqn.~\ref{eq:valid-undesirable}, the coefficient in front of $\hyper^{(0)}$ in the last term is the inner product between the training and validation gradients. Early in training, we would expect the training and validation gradients to be well-aligned, so this inner product would be positive. Hence, $\hyper$ will tend to move in the direction of $\hyper$, i.e.~away from $\mbf{0}$, simply due to the bilevel optimization dynamics. We found this to be a very strong effect in some cases, and it resulted in pathological choices of hyperparameters early in training. Using the centered parameterization appeared to fix the problem, as we show in section~\ref{sec:toy-experiments} and appendix~\ref{exp-toy}. 

\subsection{Modified Update Rule for Centered Parameterization}
\label{sec:method-modified-objective}
We observed in section~\ref{sec:method-centered-param} that in the centered parameterization, $\w_0$ can be seen as the current weights, while $\hwr$ is the Jacobian of the approximate best response function $\res_{\hw}$. This suggests a modification to the original STN training procedure. While in the original STN, the full hypernetwork was trained using sampled hyperparameters, we claim that $\w_0$ can instead be trained using gradient descent on the regularized training loss, just like the weights of an ordinary neural network. In this sense, we separate the training objectives in the following manner:
\begin{align}
    \label{eq:new-objective_hw0}
    \w_0^* = \argmin_{\w_0 \in \sR^m} &~\trainobj (\hyper, \w_0) \\
    \label{eq:new-objective_hw1}
    \mbi{\Theta}^* = \argmin_{\mbi{\Theta} \in \sR^{m \times h}} \mathbb{E}_{\pert \sim p(\pert | \scal)} &\left[\trainobj (\hyper + \pert, \res_{\hw}(\hyper + \pert, \hyper)) \right] 
\end{align}
The exclusion of perturbation in Eqn.~\ref{eq:new-objective_hw0} reduces the variance for the updates on $\w_0$, yielding faster convergence. Moreover, it can eliminate any bias to the optimal $\w_0^*$ induced by the perturbation. In appendix~\ref{sec:linear-reg-example}, we show that, even for linear regression with $L_2$ regularization, the optimal weight $\w_0^*$ does not match the correct solution under the STN's objective, whereas the modified objective recovers the correct solution. The following theorem shows that, for a general quadratic inner-level objective, the proposed parameterization converges to the best-response Jacobian. 
\begin{theorem}
\label{thm-quad-min}
Suppose $\mathcal{L}_T$ is quadratic with $\sfrac{\partial^2 \mathcal{L}_T}{\partial \w^2} \succ 0$ and let $p(\pert | \scal)$ be a diagonal Gaussian distribution with mean $\mbf{0}$ and variance $\sigma^2 \mbf{I}$. Fixing $\hyper_0 \in \mathbb{R}^h$ and $\w_0 \in \mathbb{R}^m$, the solution to the objective in Eqn.~\ref{eq:new-objective_hw1} is the best-response Jacobian.
\end{theorem}
See appendix~\ref{thm-quad-proof} for the proof.

\subsection{Direct Approximation of the Best-Response Jacobian using a Linearized Network}
\label{sec:method-linearization}
The STN aims to learn a linear hypernetwork that approximates the best-response function in a region around $\hyper$. However, if the perturbation $\pert$ is large, it may be difficult to approximate the best-response function as linear within this region. Part of the problem is that the function represented by the network behaves nonlinearly with respect to $\w$, such that the linear adjustment represented by $\boldsymbol{\Phi} \hyper$ (in Eqn.~\ref{eq:old-param}) may have a highly nonlinear effect on the network's predictions. We claim it is in fact unnecessary to account for the nonlinear effect of large changes to $\w$, as the hypernetwork is only used to estimate the best-response Jacobian at $\hyper_0$, and the Jacobian depends only on the effect of infinitesimal changes to $\w$.
\begin{wrapfigure}[15]{R}{0.42\textwidth}
    \centering
    \includegraphics[width=0.42\textwidth]{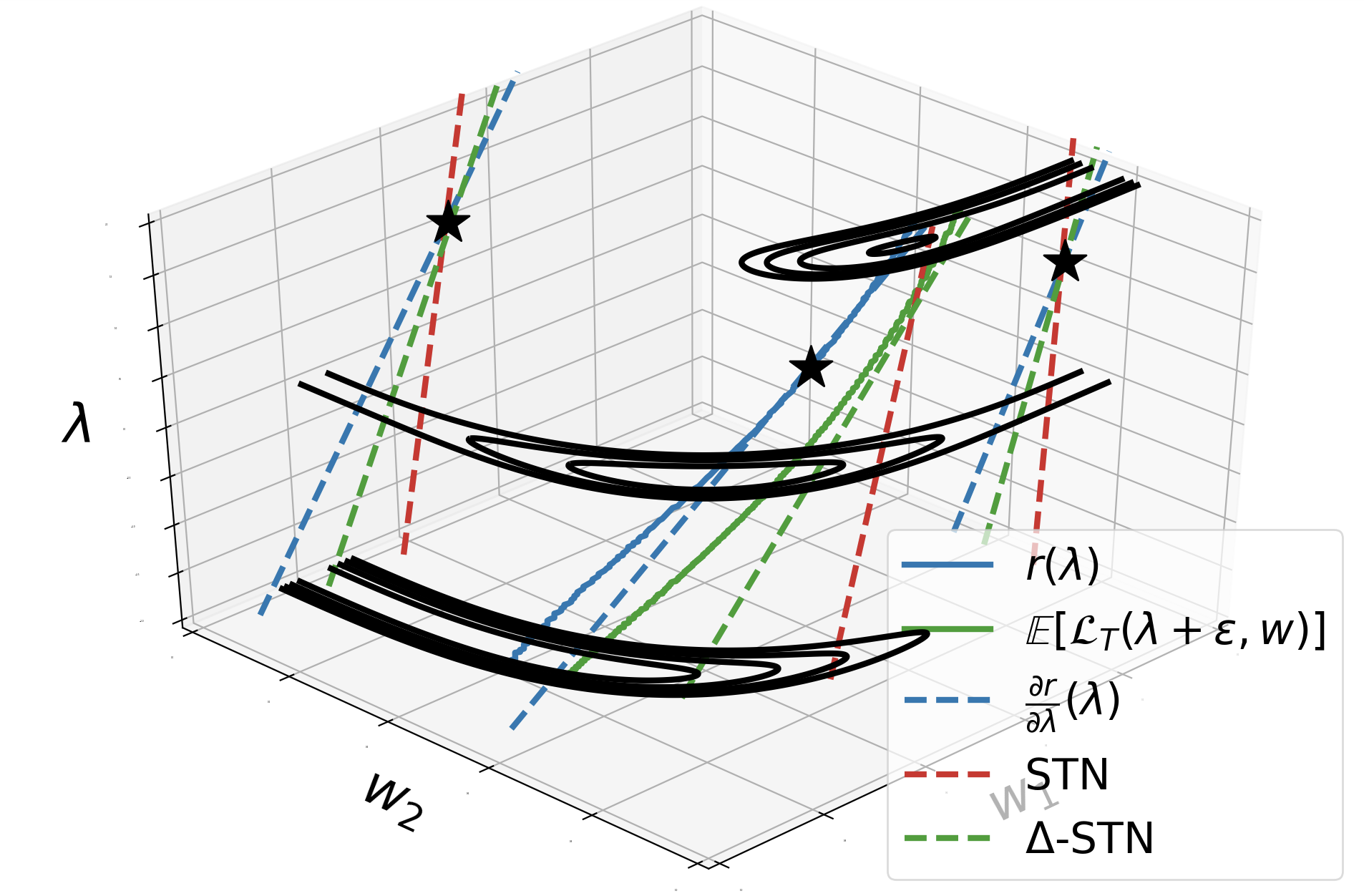}
    \vspace{-0.54cm}
    \caption{A comparison of approximated best-response Jacobians obtained by STN and $\Delta$-STN at $\lambda_0$ ($\star$). The $\Delta$-STN approximates the best-response Jacobian more accurately by linearizing the best-response hypernetwork.}
    \label{fig:toy-linearization}
\end{wrapfigure}

To remove the nonlinear dependence of the predictions on $\w$, we linearize the network around the current weights $\res(\hyper_0) = \w_0$. The first-order Taylor approximation to the network computations is given by:
\begin{equation}
    \mbf{y} = f(\bx, \w, \hyper, \xi) \approx f(\bx, \w_0, \hyper, \xi) + \jac_{\by \w} (\w - \w_0),
    \label{eq:lin-network}
\end{equation}
where $f(\bx, \w, \hyper, \xi)$ denotes evaluating a network with weights $\w$ and hyperparameters $\hyper$ on input $\bx$. Here, $\xi$ denotes a source of randomness (e.g.~dropout mask) and  $\jac_{\by\w} = \sfrac{\partial \by}{\partial \w}$ is the weight-output Jacobian. This relationship can also be written with the shorthand notation $\Delta \by \approx \jac_{\by \w} \Delta \w$, where $\Delta$ denotes a small perturbation. Therefore, we refer to it as the $\Delta$-approximation, and the corresponding bilevel optimization method as the $\Delta$-STN. In the context of our method, given the hyperparameters $\hyper_0 \in \sR^h$ and the perturbation $\pert \in \sR^h$, we structure the prediction as follows: 
\begin{align}
    \mbf{y}' = f(\bx, \res_{\mbi{\theta}}(\hyper_0 + \pert, \hyper_0), \hyper_0 + \pert, \xi) &\approx f(\bx, \res_{\mbi{\theta}}(\hyper_0, \hyper_0), \hyper_0 + \pert, \xi) + \jac_{\by \w}\Delta \w \\
    &= f(\bx, \w_0, \hyper_0 + \pert, \xi) + \jac_{\by \w} \mbi{\Theta} \pert
\end{align}

In figure~\ref{fig:toy-linearization}, we show the approximated best-response Jacobian obtained by STN and $\Delta$-STN, and compare them with the true best-response Jacobian. The contours show the loss surface of the training objective at some $\lambda$ and we projected best-response Jacobian to $(\lambda, w_1)$ and $(\lambda, w_2)$ planes for a better comparison. Because of the nonlinearity around $\lambda_0$ ($\star$), the STN tries to fit a more horizontal best-response function to minimize the error given by the perturbation, degrading the accuracy of the Jacobian approximation. On the other hand, the $\Delta$-STN linearizes the best-response function at $\lambda_0$, focusing on accurately capturing the linear effects. In appendix~\ref{app:just-linear}, we show that, with some approximations, the linearization of the perturbed inner objective yields a correct approximation of the best-response Jacobian.

Evaluating $\jac_{\by \w} \Delta \w$ in Eqn.~\ref{eq:lin-network} corresponds to evaluating the directional derivative of $f$ in the direction $\Delta \w$ and can be efficiently computed using forward mode automatic differentiation (Jacobian-vector products), a core feature in frameworks such as \texttt{JAX}~\citep{jax2018github}.

\subsection{Structured Hypernetwork Representation}
\label{section:nn}
So far, the discussion has assumed a general linear hypernetwork for simplicity. However, a full linear hypernetwork would have dimension $h \times m$, which is impractical to represent if $\hyper$ is high-dimensional. Instead, we adopt an efficient parameterization analogous to that of the original STN. Considering the $i$-th layer of the neural network, whose weights and bias are $\mbf{W}^{(i)} \in \sR^{m_i \times m_{i+1}}$ and $\mbf{b}^{(i)} \in \sR^{m_{i+1}}$, we propose to structure the layer-wise best-response hypernetwork as follows\footnote{We show the structured hypernetwork representation for convolutional layers in appendix \ref{append:conv-nn}.}:
\begin{equation}
\begin{gathered}
    \mbf{W}_{\hw}^{(i)} (\hyper, \hyper_0) = \mbf{W}_{\text{general}}^{(i)} + \left(\mbf{U}^{(i)} (\hyper - \hyper_0) \right) \odot_{\text{row}} \mbf{W}_{\text{response}}^{(i)}\\
    \mbf{b}_{\hw}^{(i)} (\hyper, \hyper_0) = \mbf{b}_{\text{general}}^{(i)} + \left(\mbf{V}^{(i)} (\hyper - \hyper_0) \right) \odot \mbf{b}_{\text{response}}^{(i)},
\end{gathered}
\end{equation}
where $\mbf{U}^{(i)}, \mbf{V}^{(i)} \in \mathbb{R}^{m_{i+1} \times h}$, and $\odot_{\text{row}}$ and $\odot$ denote row-wise and element-wise multiplications. Observe that these formulas are linear in $\hyper$, so they can be seen as a special case of Eqn.~\ref{eq:new-param}, except that structure is imposed on $\boldsymbol{\Theta}$. Observe also that this architecture is memory efficient and tractable to compute, and allows parallelism: it requires $m_{i + 1} (2 m_i + h)$ and $m_{i+1} (2 + h)$ parameters to represent the weights and bias, respectively, with 2 additional matrix multiplications and element-wise multiplications in the forward pass. Both quantities are significantly smaller than the full best-response Jacobian, so the $\Delta$-STN incurs limited memory or computational overhead compared with simply training a neural network.

\subsection{Training Algorithm}
\begin{algorithm}[H]
\caption{Training Algorithm for $\Delta$-STNs}
\SetAlgoLined
\textbf{Initialize:} hypernetwork $\hw = \{\w_0, \mbi{\Theta}\}$; hyperparameters $\boldsymbol{\lambda}$; learning rates $\{\alpha_i\}_{i=1}^3$; training and validation update intervals $T_{\text{train}}, T_{\text{valid}}$, and entropy penalty $\tau$.\\
\While{not converged}{
    \For{$t = 1, ..., T_{\text{train}}$}{
    $\boldsymbol{\epsilon} \sim p(\boldsymbol{\epsilon} | \boldsymbol{\sigma})$\\
        $\w_0 \leftarrow \w_0 - \alpha_1 \nabla_{\w_0} (\trainobj (\boldsymbol{\lambda}, \res_{\hw}(\hyper, \hyper)))$\\
        $\mbi{\Theta} \leftarrow \mbi{\Theta} - \alpha_1 \nabla_{\mbi{\Theta}} (\trainobj (\boldsymbol{\lambda} + \boldsymbol{\epsilon}, \res_{\hw}(\hyper + \pert, \hyper)))$ \Comment{Linearization with forward-mode autodiff}\\
    }
    \For{$t = 1, ..., T_{\text{valid}}$}{
        $\boldsymbol{\epsilon} \sim p(\boldsymbol{\epsilon} | \boldsymbol{\sigma})$\\
        $\boldsymbol{\lambda} \leftarrow \boldsymbol{\lambda} - \alpha_2 \nabla_{\hyper} (\valobj (\hyper + \pert, \res_{\hw}(\hyper + \pert, \hyper_0))) |_{\hyper=\hyper_0}$ \\
        $\scal \leftarrow \scal - \alpha_3 \nabla_{\scal} (\valobj (\hyper + \pert, \res_{\hw} (\hyper + \pert, \hyper)) - \tau \mathbb{H}[p(\pert | \scal)])$\\
    }
}
\label{alg:dstn-training}
\end{algorithm}
The full algorithm for $\Delta$-STNs is given in Alg.~\ref{alg:dstn-training}. In comparison to STNs (Alg.~\ref{alg:stn} in appendix~\ref{appendix:algo-stn}), we use a modified training objective and a centered hypernetwork parameterization with linearization. The hyperparameters and perturbation scale are optimized using the same objectives as in STNs (Eqn.~\ref{eq:stn-val-obj}). Since hyperparameters are often constrained (e.g.~dropout rate is in between 0 and 1), we apply a fixed non-linear transformation to the hyperparameters and optimize the hyperparameters on an unconstrained domain, as detailed in appendix~\ref{app:restricted}.

\section{Related Work}
\label{sec:related-work}
Automatic procedures for finding an effective set of hyperparameters have been a prominent subject in the literature (see~\citet{feurer2019hyperparameter} for an overview). Early works have focused on model-free approaches such as grid search and random search~\citep{bergstra2012random}. Hyperband~\citep{li2016hyperband} and Successive Halving Algorithm (SHA)~\citep{li2018massively} extend on random search by using multi-armed bandit techniques~\citep{jamieson2016non} to terminate poor-performing hyperparameter configurations early. These model-free approaches are straightforward to parallelize and work well in practice. However, they rely on random procedures, not exploiting the structure of the problem.


Bayesian Optimization (BO) provides a more principled tool to optimize the hyperparameters. Given the hyperparameters $\hyper$ and the observations $\mathcal{O} = \{ (\hyper_i, s_i)\}_{i=1}^n$, where $s$ is a surrogate loss, BO models a conditional probability $p(s | \hyper, \mathcal{O})$~\citep{hutter2010sequential,bergstra2011algorithms,snoek2012practical,snoek2015scalable}. The observations are constructed in an iterative manner, where the next set of hyperparameters to train the model is the one that maximizes an acquisition function $C(\hyper; p(s | \hyper, \mathcal{O}))$, which trades off exploitation and exploration. The training through convergence may be avoided under some assumptions on the learning curve behavior~\citep{swersky2014freeze, klein2016learning}. Nevertheless, BO requires building inductive bias into the conditional probability, is sensitive to the parameters of the surrogate model, and most importantly, does not scale well with the number of hyperparameters.

In comparison to black-box optimization, the use of gradient information can provide a drastic improvement in convergence~\citep{nesterov1998introductory}. There are two major approaches to gradient-based hyperparameter optimization. The first method uses the implicit function theorem to obtain the best-response Jacobian $\sfrac{\partial \res}{\partial \hyper}$~\citep{larsen1996design,bengio2000gradient,pedregosa2016hyperparameter,lorraine2019optimizing}, which requires approximating the Hessian (or Gauss-Newton) inverse. The second approach approximates the best-response function $\res$ by casting the inner objective as a dynamical system~\citep{domke2012generic,maclaurin2015gradient,luketina2016scalable,franceschi2018bilevel,shaban2018truncated} and applying automatic differentiation to compute the best-response Jacobian. Both approaches are computationally expensive: implicit differentiation requires approximating the inverse Hessian and unrolled differentiation needs to backpropagate through the whole gradient steps.

In contrast to implicit differentiation and unrolling, the hypernetwork approach~\citep{lorraine2018stochastic} such as Self-Tuning Networks (STNs)~\citep{mackay2019self} incurs little computation and memory overhead, as detailed in section \ref{sec:background-hypernetwork}. Moreover, it is straightforward to implement in deep learning frameworks and is able tune discrete (e.g.~number of Cutout holes~\citep{devries2017improved}) and non-differentiable (e.g.~dropout rate) hyperparameters. However, the range of applicability to general bilevel problems is slightly more restricted, as hypernetwork approach requires a single inner objective and requires that the outer variables parameterize the training objective (like implicit differentiation but unlike unrolling).


\section{Experiments}
\label{sec:experiments}
In this section, a series of experiments was conducted to investigate the following questions: (1) How does our method perform in comparison to the STN in terms of convergence, accuracy, and stability? (2) Does our method scale well to modern-size convolutional neural network? (3) Can our method be extended to other architectures such as recurrent neural networks? 

We denote our method with the centered parameterization and the modified training objective as ``centered STN'' (sections~\ref{sec:method-centered-param}, \ref{sec:method-modified-objective}), and centered STN with linearization as ``$\Delta$-STN'' (section  \ref{sec:method-linearization}).

\subsection{Toy Problems}
\label{sec:toy-experiments}
We first validated the $\Delta$-STN on linear regression and linear networks, so that the optimal weights and hyperparameters could be determined exactly. We used regression datasets from the UCI collection~\citep{Dua:2019}. For all experiments, we fixed the perturbation scale to 1, and set $T_{\text{train}} = 10$ and $T_{\text{valid}} = 1$. We compared our method with STNs and the optimal solution to the bilevel problem. We present additional results and a more detailed experimental set-up at appendix~\ref{exp-toy}.

\begin{figure}[t]
\vspace{-5mm}
\small
\centering
\includegraphics[width=1.\textwidth]{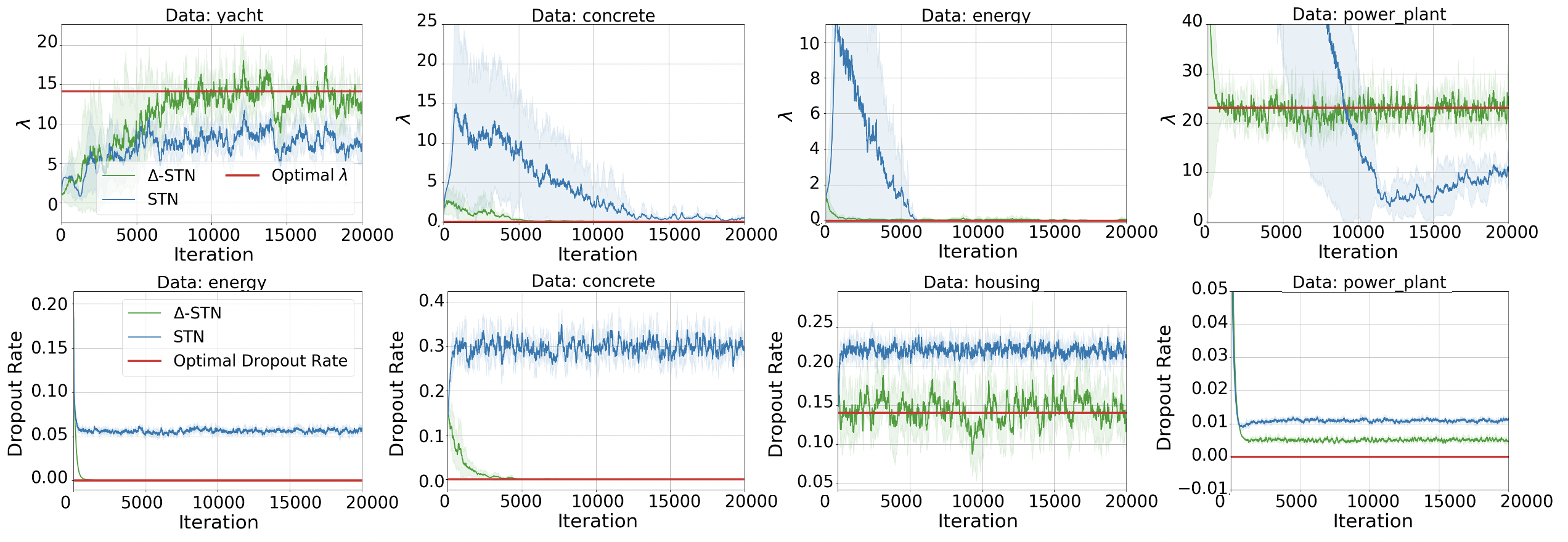}
\vspace{-3mm}
\caption{A comparison of STNs and $\Delta$-STNs on linear regression tasks (the closer to the optimal, the better). We separately optimize \textbf{(top)} weight decay and \textbf{(bottom)} input dropout rate. For both tasks, $\Delta$-STNs achieve faster convergence, higher accuracy, and improved stability compared to STNs.}
\vspace{-4mm}
\label{fig:linear-regression-result}
\end{figure}

\paragraph{Linear Regression.} We separately trained linear regression models with $L_2$ regularization and with input dropout. The trajectories for each hyperparameter are shown in~\figref{fig:linear-regression-result}. By reparameterizing the hypernetwork and modifying the training objective, the $\Delta$-STN consistently achieved faster convergence, higher accuracy, and improved stability compared to the STN.

\begin{wrapfigure}[9]{R}{0.50\textwidth}
    \centering
    \vspace{-0.6cm}
    \includegraphics[width=0.50\textwidth]{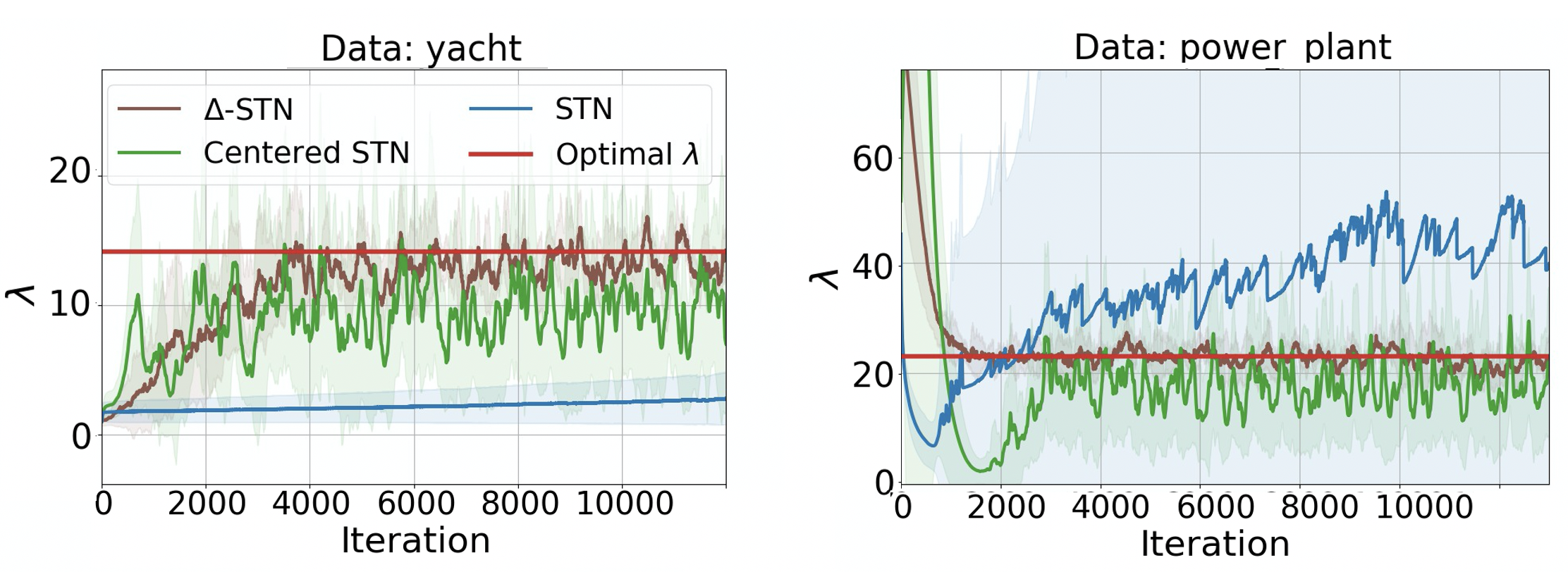}
    \vspace{-0.6cm}
    \caption{A comparison of hyperparameter updates found by STNs, centered STNs, and $\Delta$-STNs on linear network with Jacobian norm regularization.}
    \label{fig:linear-network}
\end{wrapfigure}
\paragraph{Linear Networks.} Next, we trained a 5 hidden layer linear network with Jacobian norm regularization. To show the effectiveness of linearization, we present results with and without linearizing the hypernetwork. In~\figref{fig:linear-network}, the centered STN converges to the optimal $\lambda$ more accurately and efficiently than STNs. The linearization further helped improving the accuracy and stability of the approximation.

\subsection{Image Classification}
\begin{wrapfigure}[10]{R}{0.5\textwidth}
    \centering
    \vspace{-0.5cm}
    \includegraphics[width=0.5\textwidth]{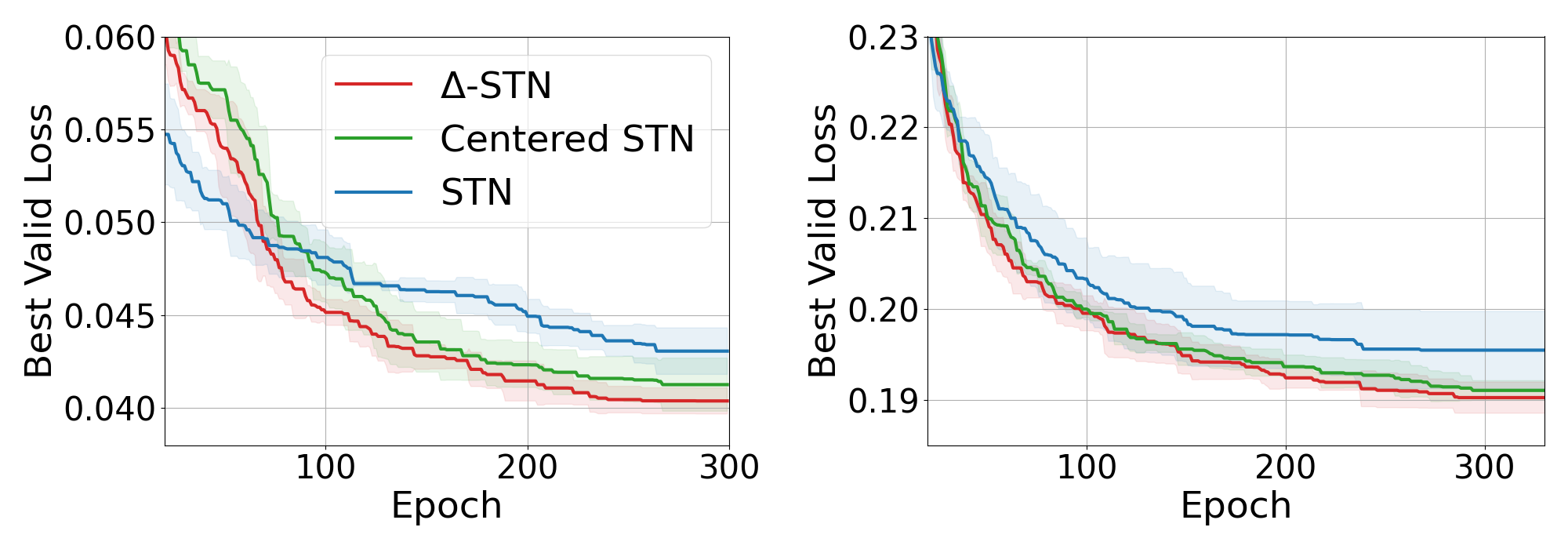}
    \vspace{-0.5cm}
    \caption{A comparison of best validation loss obtained by STNs, centered STNs, and $\Delta$-STNs for \textbf{(left)} MNIST and \textbf{(right)} FashionMNIST datasets. }
    \label{fig:mnist-result}
\end{wrapfigure}
To evaluate the scalability of our proposed architecture to deep neural networks, we applied $\Delta$-STNs to image classification tasks. We set $T_{\text{train}} = 5$ and $T_{\text{valid}} = 1$ for all experiments and compared $\Delta$-STNs to random search (RS)~\citep{bergstra2012random}, Bayesian optimization (BO)~\citep{snoek2012practical, snoek2015scalable}, and Self-Tuning Networks (STNs)~\citep{mackay2019self}. The final performances on validation and test losses are summarized in table~\ref{tab:classification}. Our $\Delta$-STN achieved the best generalization performance for all experiments, demonstrating the effectiveness of our approach. The details of the experiment settings and additional results are provided in appendix~\ref{exp-image-class}. Moreover, we show that $\Delta$-STNs are more robust to hyperparameter initialization and perturbation scale in appendix~\ref{appendix:additional-results}. 

\paragraph{MNIST \& FashionMNIST.} 
We applied $\Delta$-STNs on MNIST~\citep{lecun1998gradient} and FashionMNIST~\citep{xiao2017/online} datasets. For MNIST, we trained a multilayer perceptron with 3 hidden layers of rectified units~\citep{nair2010rectified,glorot2011deep}. We tuned 3 dropout rates that control the input and per-layer activations. For FashionMNIST, we trained a convolutional neural network composed of two convolution layers with 32 and 64 filters, followed by 2 fully-connected layers. In total 6 hyperparameters were optimized: input dropout, per-layer dropouts, and Cutout holes and length~\citep{devries2017improved}. Both networks were trained for 300 epochs and a comparison was made between STNs and $\Delta$-STNs in terms of the best validation loss achieved by a given epoch as shown in~\figref{fig:mnist-result}. The $\Delta$-STNs was able to achieve a better generalization with faster convergence. 

\paragraph{CIFAR-10.} Finally, we evaluated $\Delta$-STNs on the CIFAR-10 dataset~\citep{krizhevsky2009learning}. We used the AlexNet~\citep{krizhevsky2012imagenet}, VGG16~\citep{simonyan2014very}, and ResNet18~\citep{he2016deep} architectures. For all architectures, we tuned (1) input dropout, (2) per-layer dropouts, (3) Cutout holes and length, and (4) amount of noise applied to hue, saturation, brightness and contrast to the image, (5) random scale, translation, shear, rotation applied to the image, resulting in total of 18, 26, and 19 hyperparameters. Figure~\ref{fig:alexnet} shows the best validation loss achieved by each method over time and the hyperparameter schedules prescribed by $\Delta$-STNs for AlexNet. $\Delta$-STNs achieved the best generalization performance compared to other methods, in less time.
\begin{figure}[t]
     \subfigure[Time comparison]{\includegraphics[width=0.33\textwidth]{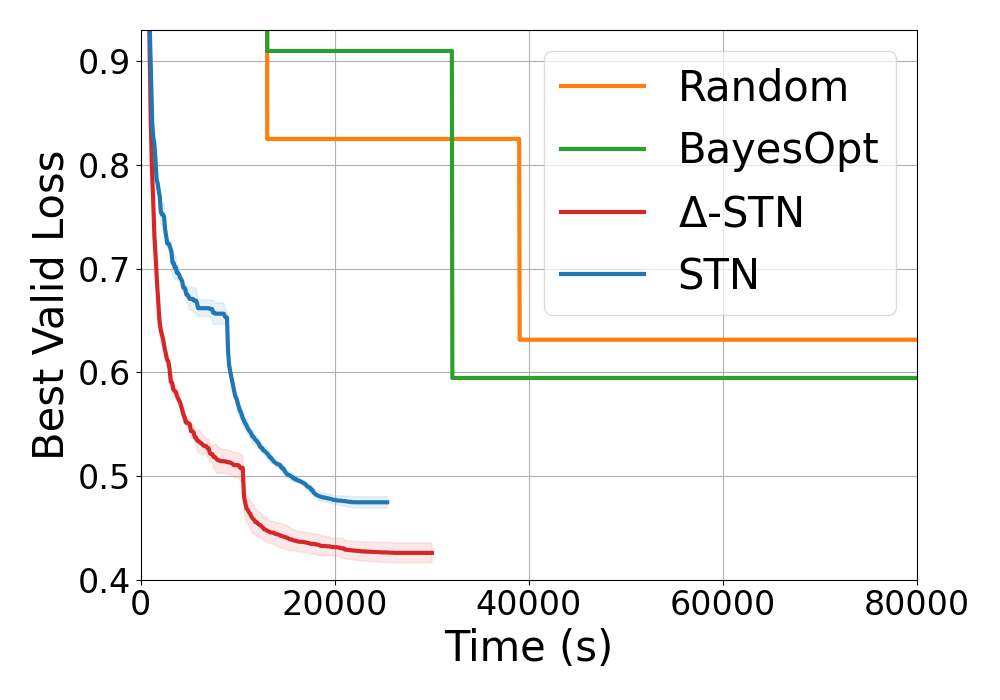}\label{fig:alexnet-time}}
     \subfigure[Schedule for dropouts]{\includegraphics[width=0.33\textwidth]{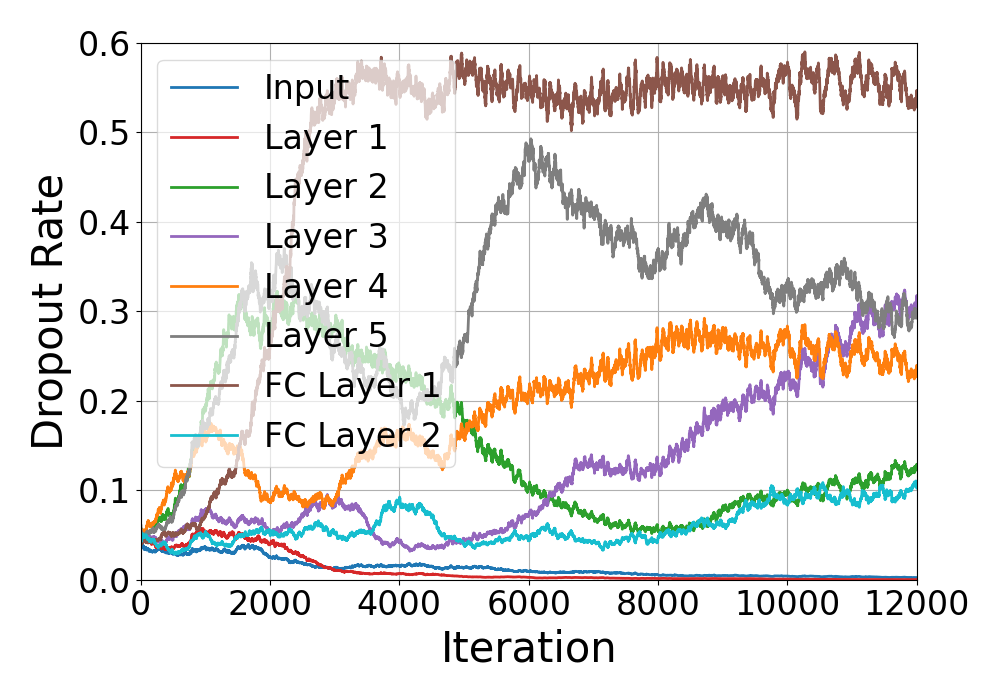}\label{fig:alexnet-dropout}}
     \subfigure[Schedule for data augmentation]{\includegraphics[width=0.33\textwidth]{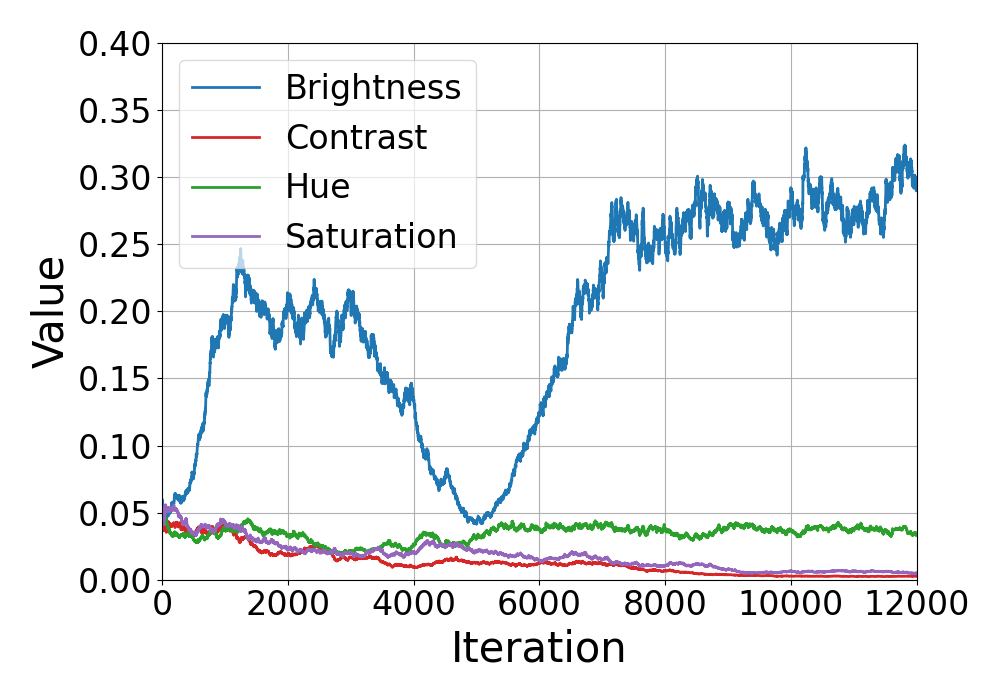}\label{fig:alexnet-data-aug}}
     \caption{\textbf{(a)} A comparison of the best validation loss achieved by random search, Bayesian optimization, STNs, and $\Delta$-STNs over time for AlexNet. $\Delta$-STNs achieved the lowest validation loss more efficiently than other methods. \textbf{(b)}, \textbf{(c)}: Hyperparameter schedules found by $\Delta$-STNs for dropout rates and data augmentation parameters.}
     \label{fig:alexnet}
     \vspace{-2mm}
\end{figure}
 
\begin{center}
\begin{table}[t]
\small
\caption{Final validation (test) losses / perplexities of each method on the image classification tasks and the PTB word-level language modeling task.}
\label{tab:classification}
\begin{center}
\centering
\resizebox{\textwidth}{!}{
\begin{tabular}{l|l|c|c|c|c|c}
\toprule
Dataset & Network & RS & BO & STN & Centered STN & $\Delta$-STN\\ \hline
MNIST & MLP  & 0.043 (0.042) & 0.042 (0.043) & 0.043  (0.041) & 0.041  (0.039) & \textbf{0.040} (\textbf{0.038}) \\ \hline
FMNIST & SimpleCNN & 0.206 (0.214) & 0.217 (0.215) & 0.196 (0.218) & 0.191 (0.212) & \textbf{0.189} (\textbf{0.209}) \\ \hline
\multirow{3}{*}{CIFAR10} & AlexNet & 0.631 (0.671) & 0.594 (0.598) & 0.474 (0.488) & 0.431 (0.450) & \textbf{0.425} (\textbf{0.446}) \\ \cline{2-7} 
& VGG16 &  0.566 (0.595) & 0.421 (0.446) & 0.330 (0.354)&  0.286 (0.321)       &    \textbf{0.272} (\textbf{0.296}) \\ \cline{2-7} 
& ResNet18 & 0.264 (0.298)  & 0.230 (0.267) & 0.266 (0.312)&  0.222 (0.258)       &    \textbf{0.204} (\textbf{0.238}) \\ \cline{2-7} \hline
PTB & LSTM & 84.81 (81.46) & 72.13 (69.29) & 70.67 (67.78) & 69.40 (66.67) & \textbf{68.63} (\textbf{66.26}) \\
\bottomrule
\end{tabular}
}
\end{center}
\vspace{-0.3cm}
\end{table}
\end{center}

\subsection{Language Modelling}
To demonstrate that $\Delta$-STNs can be extended to different settings, we applied our method to an LSTM network on the Penn Tree Bank (PTB) dataset~\citep{marcus1993building}, which has been popularly used for RNN regularization benchmarks because of its small size~\citep{gal2016theoretically,merity2017regularizing,wen2018flipout,mackay2019self}. We trained a model consisted of 2 LSTM layers and tuned 7 hyperparameters: (1) variational dropout on the input, hidden state, and output, (2) embedding dropout, (3) Drop-Connect~\citep{wan2013regularization}, and (4) activation regularization coefficients. A more detailed explanation on role of each hyperparameter and the experiment set-up can be found in appendix~\ref{exp:language}. The final validation and test perplexities achieved by each method are shown in Table~\ref{tab:classification}. The $\Delta$-STN outperforms other baselines, achieving the lowest validation and test perplexities.

\section{Conclusion}
We introduced $\Delta$-Self-Tuning Networks ($\Delta$-STNs), an improved hypernetwork architecture that efficiently optimizes hyperparameters online. We showed that $\Delta$-STNs fix subtle pathologies in training STNs by (1) reparameterizing the hypernetwork, (2) modifying the training objectives, and (3) linearzing the best-response hypernetwork. The key insight was to accurately approximate the best-response Jacobian instead of the full best-response function. Empirically, we demonstrated that $\Delta$-STNs achieve better generalization in less time, compared to existing approaches to bilevel optimization. We believe that $\Delta$-STNs offer a more reliable, robust, and accurate deployment of hyperparameter optimization based on hypernetwork approach, and offer an alternative method to efficiently solve bilevel optimization problems. 

\newpage
\section*{Acknowledgements and Funding Disclosure}
This work was supported by a grant from LG Electronics and a Canada Research Chair. RG acknowledges support from the CIFAR Canadian AI Chairs program. Part of this work was carried out when RG was a visitor at the Institute for Advanced Study in Princeton, NJ. We thank Xuchan Bao, David Duvenaud, Jonathan Lorraine, Matthew MacKay, and Paul Vicol for helpful discussions. 

\section*{Broader Impact}
Most application of deep learning involves regularization hyperparameters, and hyperparameter tuning is one stage of a much longer pipeline. Hence, any discussion of societal impacts would necessarily be speculative. One predictable impact of this work is to lessen the need for massive computing resources to tune hyperparameters.

\bibliographystyle{abbrvnat}
\bibliography{neurips_2020}

\normalsize
\appendix
\newpage

\section{Table of Notation}
\label{appendix:notation}

\begin{table}[h]
    \centering
    \noindent\setlength\tabcolsep{4pt}\setlength{\extrarowheight}{5pt}
    \begin{tabularx}{0.9\linewidth}{c|*{1}{>{\raggedright\arraybackslash}X}}
         \textbf{Notation} & \textbf{Description} \\
         \hline
        $\hyper, \w$ & Hyperparameters and parameters\\
        $\mathcal{L}_T, \mathcal{L}_V$ & Training and validation objectives\\
        $h$ & Number of hyperparameters \\
        $m$ & Number of parameters\\
        $p$ & Number of hypernetwork parameters\\
        $n$ & Number of training data\\
        $v$ & Number of validation data\\
        $\res(\hyper)$ & Best-response function\\
        $\sfrac{\partial \res}{\partial \hyper}$ & Best-response Jacobian\\
        $\pert$ & Perturbation vector\\
        $\scal$ & Perturbation scale\\
        $p(\pert | \scal)$ & Gaussian distribution with mean $\mbf{0}$ and covariance $\text{diag}(\scal)$ \\
        $\mbi{\phi}$ & STN's parameters\\
        $\res_{\mbi{\phi}}$ & Parametric best-response function with STN's parameterization\\
        $\mathbb{H}$ & Entropy term \\
        $\tau$ & Entropy weight \\
        $\mbi{\theta}$ & $\Delta$-STN's parameters \\
        $\res_{\mbi{\theta}}$ & Parametric best-response Jacobian with $\Delta$-STN's parameterization \\
        $\jac_{\by \w}$ &  Weight-output Jacobian \\
        $\mbf{H}$ & Hessian \\
        $\mbf{G}$ & Gauss-Newton Hessian \\
        $\kappa$ & Condition number \\
        $\otimes$ & Kronecker product \\
        $\alpha$ & Learning rate \\
        $\mbf{y}$ & Network's prediction \\
        $f$ & Network function\\
        $\odot, \odot_{\text{row}}$ & Element-wise multiplication and row-wise multiplication \\
        $T_{\text{train}}, T_{\text{valid}}$ & Training and validation update intervals \\
        $\mbf{X}, \mbf{t}$ & Training data and target vector \\
        $\w^*$ & Optimal parameters \\
        $s$ & Fixed non-linear transformation on hyperparameters \\
        \hline
    \end{tabularx}
	\vspace{1em}
    \caption{A summary of notations used in this paper.}
    \label{tab:notation}
\end{table}

\newpage

\section{Training Algorithm for Self-Tuning Networks (STNs)}
\label{appendix:algo-stn}
In this section, we present the training algorithm for Self-Tuning Networks. The full algorithm is shown in Alg.~\ref{alg:stn}.

\begin{algorithm}[H]
\label{alg:method}
\SetAlgoLined
\textbf{Initialize:} Best-response parameters $\boldsymbol{\phi} = \{\mbi{\phi}_0, \mbi{\Phi} \}$; hyperparameters $\boldsymbol{\lambda}$; learning rates $\{\alpha_i\}_{i=1}^3$; training and validation update intervals $T_{\text{train}}$ and $T_{\text{valid}}$; entropy weight $\tau$.\\
 \While{not converged}{
    \For{$t = 1, ..., T_{\text{train}}$}{
             $\boldsymbol{\epsilon} \sim p(\boldsymbol{\epsilon} | \boldsymbol{\sigma})$\\
    $\boldsymbol{\phi} \leftarrow \boldsymbol{\phi} - \alpha_1 \nabla_{\mbi{\phi}} (\mathcal{L}_T (\boldsymbol{\lambda} + \pert, r_{\boldsymbol{\phi}}(\hyper + \pert)))$ \\
    }
    \For{$t = 1, ..., T_{\text{valid}}$}{
        $\boldsymbol{\epsilon} \sim p(\boldsymbol{\epsilon} | \boldsymbol{\sigma})$\\
        $\hyper \leftarrow \boldsymbol{\lambda} - \alpha_2 \nabla_{\hyper} (\mathcal{L}_V (\hyper + \pert, r_{\boldsymbol{\phi}}(\hyper + \pert)))$\\
        $\scal \leftarrow \scal - \alpha_3 \nabla_{\scal} (\mathcal{L}_V (\hyper + \pert, r_{\boldsymbol{\phi}}(\hyper + \pert)) - \tau \mathbb{H}[p(\boldsymbol{\epsilon} | \boldsymbol{\sigma})])$
    }
}
\caption{Training Algorithm for Self-Tuning Networks (STNs)~\citep{mackay2019self}}
\label{alg:stn}
\end{algorithm}

\section{Proofs}
\subsection{Observation~\ref{obs:gnhhw}}
\label{obs:gnhhw-proof}
\newtheorem*{O1}{Observation~\ref{obs:gnhhw}}
\begin{O1}
The Gauss-Newton Hessian with respect to the hypernetwork is given by:
\begin{equation}
    \gnHess_\stnhwVec = \mathbb{E} \left[\hat{\hyper} \hat{\hyper}^{\top} \otimes \jac_{\by\w}^{\top} \hess_\by \jac_{\by\w} \right],
\end{equation}
where $\hat{\hyper} = \bar{\hyper} + \bar{\pert}$ are the sampled hyperparameters and $\bar{\pert}$ is the perturbation vector appended with additional homogeneous coordinate with value 0.
\end{O1} 
\begin{proof}
The hypernetwork-output Jacobian is as follows:
\begin{align}
\mbf{J}_{\mbf{y} \stnhwVec} = \hat{\hyper}^{\top} \otimes \jac_{\by\w}
\end{align}
Then, the Gauss-Newton Hessian with respect to the hypernetwork is:
\begin{align}
\mbf{G}_{\stnhwVec} &= \mathbb{E} \left[\mbf{J}_{\mbf{y} \stnhwVec}^{\top} \mbf{H}_{\mbf{y}} \mbf{J}_{\mbf{y} \stnhwVec} \right]\\
&= \mathbb{E}[(\hat{\hyper} \otimes \jac_{\by\w}^{\top})(\mbf{1} \otimes \mbf{H}_{\mbf{y}})(\hat{\hyper}^{\top} \otimes \jac_{\by\w})]\\
&= \mathbb{E}[\hat{\hyper} \hat{\hyper}^{\top} \otimes \jac_{\by\w}^{\top} \mbf{H}_{\mbf{y}} \jac_{\by\w}],
\end{align}
where $\mbf{1}$ is a $1 \times 1$ matrix with entry 1.
\end{proof}

\subsection{Observation~\ref{obs:condition-kfac}}
\label{obs:condition-kfac-proof}
\newtheorem*{O2}{Observation~\ref{obs:condition-kfac}}
\begin{O2}
Let $\kappa(\mbf{A})$ denote the condition number of a square positive definite matrix $\mbf{A}$. Given square positive definite matrices $\mbf{A}$ and $\mbf{B}$, the condition number of $\mbf{A} \otimes \mbf{B}$ is given by $\kappa ( \mbf{A} \otimes \mbf{B} )=\kappa(\mbf{A}) \kappa(\mbf{B})$. 
\end{O2} 
\begin{proof}
Let $\mbf{A}$ and $\mbf{B}$ be square positive definite matrices. Then, we get:
\begin{align}
    \kappa(\mbf{A} \otimes \mbf{B}) &= \norm{\mbf{A} \otimes \mbf{B}} \norm{\mbf{A}^{-1} \otimes \mbf{B}^{-1}} \\
    &= \norm{\mbf{A}} \norm{\mbf{B}} \norm{\mbf{A}^{-1}} \norm{\mbf{B}^{-1}} \\
    &= \kappa(\mbf{A}) \kappa(\mbf{B})
\end{align}
\end{proof}

\subsection{Theorem \ref{thm-quad-min}}
\label{thm-quad-proof}

\newtheorem*{T1}{Theorem~\ref{thm-quad-min}}
\begin{T1}
  Suppose $\mathcal{L}_T$ is quadratic with $\sfrac{\partial^2 \mathcal{L}_T}{\partial \w^2} \succ 0$ and let $p(\pert | \scal)$ be a diagonal Gaussian distribution with mean $\mbf{0}$ and variance $\sigma^2 \mbf{I}$. Fixing $\hyper_0 \in \mathbb{R}^h$ and $\w_0 \in \mathbb{R}^m$, the solution to the objective in Eqn.~\ref{eq:new-objective_hw1} is the best-response Jacobian.
\end{T1} 
\begin{proof}

Consider a quadratic function $\mathcal{L}_T \colon \mathbb{R}^h \times \mathbb{R}^m \to \mathbb{R}$ defined as:
\begin{align}
    \mathcal{L}_T(\hyper, \w) &= \frac{1}{2} \begin{pmatrix} 
\w^{\top} & \hyper^{\top}
\end{pmatrix}
\begin{pmatrix} 
\mbf{A} & \mbf{B} \\
\mbf{B}^{\top} & \mbf{C} 
\end{pmatrix}
\begin{pmatrix} 
\w  \\
\hyper  
\end{pmatrix} + \mbf{d}^{\top} \w + \mbf{e}^{\top} \hyper + c,
\end{align}
where $\mbf{A} \in \mathbb{R}^{m \times m}, \mbf{B}  \in \mathbb{R}^{m \times h}, \mbf{C}  \in \mathbb{R}^{h \times h}, \mbf{d}  \in \mathbb{R}^{m}, \mbf{e} \in \mathbb{R}^{h}$, and $c \in \mathbb{R}$. By assumption, $\mbf{A}$ is a positive definite matrix. Fixing $\hyper_0 \in \mathbb{R}^h$, the optimal weight is given by:
\begin{align}
    \w^* = - \mbf{A}^{-1} (\mbf{B} \hyper_0 + \mbf{d})
\end{align}
and the best-response function is:
\begin{align}
    \res (\hyper) = -\mbf{A}^{-1} (\mbf{B} \hyper + \mbf{d})
\end{align}
Consequently, the best-response Jacobian is as follows:
\begin{align}
    \frac{\partial \res}{\partial \hyper} (\hyper_0) = - \mbf{A}^{-1} \mbf{B} \in \mathbb{R}^{m \times h}
\end{align}
Given $\w_0 \in \mathbb{R}^m$, formulating the objective as in Eqn~\ref{eq:new-objective_hw1}, we have:
\begin{align}
    &\mathbb{E}_{\pert \sim p(\pert | \scal)} \left[ \mathcal{L}_T (\hyper_0 + \pert, \w_0 + \mbi{\Theta} \pert) \right]\\
    &= \mathbb{E}_{\pert \sim p(\pert | \scal)} \left[\frac{1}{2} (\w_0 + \mbi{\Theta} \pert)^{\top} \mbf{A} (\w_0 + \mbi{\Theta} \pert)  + (\w_0 + \mbi{\Theta} \pert)^{\top} \mbf{B} (\hyper_0 + \pert) \right. \\
    &\quad\quad\quad\quad\quad\quad \left. + \frac{1}{2}(\hyper_0 + \pert)^{\top} \mbf{C} (\hyper_0 + \pert) + \mbf{d}^{\top} (\w_0 + \mbi{\Theta} \pert) + \mbf{e}^{\top} (\hyper_0+ \pert) + c \vphantom{\frac{1}{2}} \right] \\
    &= \mathbb{E}_{\pert \sim p(\pert | \scal)} \left[\text{\circled{1}} +\text{\circled{2}} + \text{\circled{3}} + \text{\circled{4}} \right],
\end{align}
where each component is:
\begin{align}
    \text{\circled{1}} &= \frac{1}{2} (\pert^{\top} \mbi{\Theta}^{\top} \mbf{A} \mbi{\Theta} \pert + 2 \pert^{\top} \mbi{\Theta}^{\top} \mbf{A} \w_0 + \w_0^{\top} \mbf{A} \w_0) \\
    \text{\circled{2}} &= \w_0^{\top} \mbf{B} \hyper_0 + \w_0^{\top} \mbf{B} \pert + \pert^{\top} \mbi{\Theta}^{\top} \mbf{B} \hyper_0 + \pert^{\top} \mbi{\Theta}^{\top} \mbf{B} \pert\\
    \text{\circled{3}} &= \frac{1}{2} (\hyper_0^{\top} \mbf{C} \hyper_0 + 2 \pert^{\top} \mbf{C} \hyper_0 + \pert^{\top} \mbf{C} \pert) \\
    \text{\circled{4}} &= \mbf{d}^{\top} \w_0 + \mbf{d}^{\top} \mbi{\Theta} \pert + \mbf{e}^{\top} \hyper_0 + \mbf{e}^{\top} \pert + c
\end{align}
We simplify these expressions by using linearity of expectation and using the fact that $\mathbb{E}[\pert^{\top} \pert] = \sigma^2$:
\begin{align}
    \mathbb{E}_{\pert \sim p(\pert | \scal)} \left[\text{\circled{1}} \right] &= \frac{1}{2} (\text{Tr}[ \mbi{\Theta}^{\top} \mbf{A} \mbi{\Theta}]  \sigma^2+ \w_0^{\top} \mbf{A} \w_0) \\
    \mathbb{E}_{\pert \sim p(\pert | \scal)} \left[\text{\circled{2}} \right] &= \w_0^{\top} \mbf{B} \hyper_0 + \text{Tr}[\mbi{\Theta}^{\top} \mbf{B}] \sigma^2 \\
   \mathbb{E}_{\pert \sim p(\pert | \scal)} \left[\text{\circled{3}} \right] &= \frac{1}{2} ( \hyper_0^{\top} \mbf{C} \hyper_0 + \text{Tr}[\mbf{C}] \sigma^2)\\
    \mathbb{E}_{\pert \sim p(\pert | \scal)} \left[\text{\circled{4}} \right] &= \mbf{d}^{\top} \w_0 + \mbf{e}^{\top} \hyper_0 + c
\end{align}
Then, the gradient with respect to $\mbi{\Theta}$ is:
\begin{align}
    \frac{\partial }{\partial \mbi{\Theta}} \left( \mathbb{E}_{\pert \sim p(\pert | \scal)} \left[\mathcal{L}_T (\hyper_0 + \pert, \w_0 + \mbi{\Theta} \pert) \right] \right) &= \mbf{B} \sigma^2 + \mbf{A} \mbi{\Theta} \sigma^2
\end{align}
Setting the above equation equal to $\mbf{0}$, the optimal solution $\hwr^*$ is the following:
\begin{align}
    \mbi{\Theta}^* &= - \mbf{A}^{-1} \mbf{B} = \frac{\partial \res}{\partial \hyper} (\hyper_0),
\end{align}
which matches the best-response Jacobian.
\end{proof}

\subsection{Lemma~\ref{best-response-lamma}}

\begin{lemma}
\label{best-response-lamma}
Let $\hyper_0 \in \sR^h$ and choose $\w_0 \in \sR^m$ to be the solution to Eqn.~\ref{eq:new-objective_hw0}. Suppose $\trainobj \in \mathcal{C}^2$ in a neighborhood of $(\hyper_0, \w_0)$ and the Hessian $\sfrac{\partial \trainobj}{\partial \w} (\hyper_0, \w_0)$ is positive definite. Then, for some neighborhood $\mathcal{U}$ of $\hyper_0$, there exists a unique continuously differentiable function $\res \colon \mathcal{U} \to \sR^m$ such that $\res(\hyper_0) = \w_0$ and $\sfrac{\partial \mathcal{L}_T}{\partial \w} (\hyper, \res(\hyper)) = \mbf{0}$ for all $\hyper \in \mathcal{U}$. Moreover, the best-response Jacobian on $\mathcal{U}$ is as follows:
\begin{align}
    \frac{\partial \res}{\partial \hyper} (\hyper) = -  \left[ \frac{\partial^2 \mathcal{L}_T}{\partial \w^2} (\hyper, \res (\hyper)) \right]^{-1} \left( \frac{\partial^2 \mathcal{L}_T}{\partial \w \partial \hyper} (\hyper, \res(\hyper)) \right)
\end{align}
\end{lemma}
\begin{proof}
Let $\hyper_0 \in \mathbb{R}^h$ and $\w_0 \in \mathbb{R}^m$ be the solution to Eqn.~\ref{eq:new-objective_hw0}. Suppose $\mathcal{L}_T$ is $\mathcal{C}^2$ in a neighborhood of $(\hyper_0, \w_0)$. By first-order optimality condition, we have:
\begin{align}
    \frac{\partial \mathcal{L}_T}{\partial \w} (\hyper_0, \w_0) = \mbf{0}
\end{align}
Since the Hessian is positive definite, it is invertible, and there exists a unique continuously differentiable function $\res \colon \mathcal{U} \to \sR^m$ for some neighborhood $\mathcal{U}$ of $\hyper_0$ such that $\res(\hyper_0) = \w_0$ and:
\begin{align}
    \frac{\partial \mathcal{L}_T}{\partial \w} (\hyper, \res(\hyper)) = \mbf{0}
\end{align}
for all $\hyper \in \mathcal{U}$ by implicit function theorem. Also, we have:
\begin{align}
    \mbf{0} &= \frac{\mathrm{d}}{\mathrm{d} \hyper}  \left(\frac{\partial \mathcal{L}_T}{\partial \w} (\hyper, \res(\hyper)) \right) \\
    &= \left( \frac{\partial^2 \mathcal{L}_T}{\partial \w^2} (\hyper, \res(\hyper)) \frac{\partial \res}{\partial \hyper} (\hyper)  +  \frac{\partial^2 \mathcal{L}_T}{\partial \w \partial \hyper} (\hyper, \res(\hyper)) \right)^{\top}
    \label{eq:deriv-best-response-jacob}
\end{align}
for all $\hyper \in \mathcal{U}$. Re-arranging Eqn.~\ref{eq:deriv-best-response-jacob}, we get:
\begin{align}
    \frac{\partial \res}{\partial \hyper} (\hyper) = - \left[ \frac{\partial^2 \mathcal{L}_T}{\partial \w^2} (\hyper, \res(\hyper)) \right]^{-1} \left( \frac{\partial^2 \mathcal{L}_T}{\partial \hyper \partial \w} (\hyper, \res(\hyper)) \right)
\end{align}
\end{proof}

\section{Justification for Linearizing the Best-Response Hypernetwork}
\label{app:just-linear}
Consider the inner-level objective:
\begin{align}
    \res (\hyper) = \argmin_{\w \in \mathbb{R}^m} \mathbb{E}_{\pert \sim p(\pert | \scal)} \left[ \mathcal{L}_T (\hyper + \pert, \w) \right]
\end{align}
Let $\w_0 = \res (\hyper_0)$ be the current weights and assume it is the optimal solution. Further assuming we can exchange the integral and the gradient operator, by first-order optimality condition, we get:
\begin{align}
    \frac{\partial}{\partial \w} \mathbb{E}_{\pert \sim p(\pert | \scal)} \left[ \mathcal{L}_T (\hyper_0 + \pert, \w_0) \right] = \mathbb{E}_{\pert \sim p(\pert | \scal)} \left[\frac{\partial \mathcal{L}_T}{\partial \w}  (\hyper_0 + \pert, \w_0) \right] = \mbf{0}
\end{align}
Differentiating with respect to $\hyper$, we have:
\begin{align}
    \frac{\mathrm{d}}{\mathrm{d} \hyper} \mathbb{E}_{\pert \sim p(\pert | \scal)} \left[\frac{\partial \mathcal{L}_T}{\partial \w}  (\hyper + \pert, \res(\hyper)) \right] = \mathbb{E}_{\pert \sim p(\pert | \scal)} \left[\frac{\mathrm{d}}{\mathrm{d} \hyper} \left( \frac{\partial  \mathcal{L}_T }{\partial \w }(\hyper + \pert, \res(\hyper)) \right)\right] = \mbf{0}
\end{align}
Then:
\begin{align}
    \mathbb{E}_{\pert \sim p(\pert | \scal)} \left[ \frac{\partial^2 \mathcal{L}_T}{\partial \w \partial \hyper}  (\hyper_0 + \pert, \w_0) +  \frac{\partial^2 \mathcal{L}_T}{\partial \w^2}  (\hyper_0 + \pert, \w_0) \frac{\partial \res}{\partial \hyper} (\hyper_0)  \right] = \mbf{0}
    \label{eq:brj1}
\end{align}
For simplicity, we denote:
\begin{align}
    \mbf{B}(\hyper_0, \w_0, \pert) &= \frac{\partial^2 \mathcal{L}_T}{\partial \w \partial \hyper} (\hyper_0 + \pert, \w_0) \in \mathbb{R}^{m \times h} \\
    \mbf{A} (\hyper_0, \w_0, \pert) &= \frac{\partial^2 \mathcal{L}_T}{\partial \w^2} (\hyper_0 + \pert, \w_0) \in \mathbb{R}^{m \times m} \\
    \hwr &=  \frac{\partial \res}{\partial \hyper} (\hyper_0) \in \mathbb{R}^{m \times h}
\end{align}
Thus, $\mbi{\Theta}$ is the best-response Jacobian, and it is given by:
\begin{align}
    \mbi{\Theta} = - \mathbb{E}[\mbf{A}]^{-1} \mathbb{E}[\mbf{B}]
\end{align}
We can represent the solution as a minimization problem:
\begin{align}
    \hwr^* = \argmin_{\hwr \in \mathbb{R}^{m \times h}} \mathbb{E}_{\pert \sim p(\pert | \scal)} \left[\frac{1}{2} \text{tr}[\mbf{A} \hwr \hwr^{\top}] + \text{tr}[\mbf{B}^{\top} \hwr]\right]
\end{align}
The first term in Eqn.~\ref{eq:brj1} can be represented as:
\begin{align}
    \mathbb{E}_{\pert \sim p(\pert | \scal)} \left[\mbf{B}(\hyper_0, \w_0, \pert) \right] &= \mathbb{E}_{\pert \sim p(\pert | \scal)} \left[ \frac{\partial^2 \mathcal{L}_T}{\partial \w \partial \hyper} (\hyper_0 + \pert, \w_0) \right] \\
    &= \mathbb{E}_{\tilde{\pert} \sim p(\tilde{\pert} | \mbf{I})} \left[ \frac{\partial^2 \tilde{\mathcal{L}_T}}{\partial \w \partial \tilde{\pert}} (\tilde{\pert}, \w_0) \scalm^{- 1/2}\right] \\
    \label{app:stein}
    &= \mathbb{E}_{\tilde{\pert} \sim p(\tilde{\pert} | \mbf{I})} \left[ \left(\frac{\partial \tilde{\mathcal{L}_T}}{\partial \w} (\tilde{\pert}, \w_0) \right)^{\top} \tilde{\pert}^{\top} \scalm^{-1/2} \right] \\
    &= \mathbb{E}_{\pert \sim p(\pert | \scal)} \left[ \left(\frac{\partial \mathcal{L}_T}{\partial \w} (\hyper_0 + \pert, \w_0) \right)^{\top} \pert^{\top} \scalm^{-1} \right],
\end{align}
where $\pert = \scalm^{1/2} \tilde{\pert}$ and $\mathcal{L}_T (\hyper_0 + \scalm^{1/2} \tilde{\pert}, \w)= \tilde{\mathcal{L}_T} (\tilde{\pert}, \w)$. The third line (Eqn.~\ref{app:stein}) uses Stein's identity. Multipying $\scalm$ in Eqn.~\ref{eq:brj1}, we have:
\begin{align}
    \mathbb{E}_{\pert \sim p(\pert | \scal)} \left[\mbf{B}(\hyper_0, \w_0, \pert) \scalm + \mbf{A}(\hyper_0, \w_0, \pert) \hwr \scalm \right] = \mbf{0}
\end{align}
with the optimization problem:
\begin{align}
    \hwr^* = \argmin_{\hwr \in \mathbb{R}^{m \times h}} \mathbb{E}_{\pert \sim p(\pert | \scal)} \left[\frac{1}{2} \text{tr}[\mbf{A} \hwr \scalm \hwr^{\top}] + \text{tr}[\mbf{B}^{\top} \hwr \scalm]\right]
    \label{eq:second-stein-lem}
\end{align}
The second term in Eqn~\ref{eq:second-stein-lem} is:
\begin{align}
    \mathbb{E}_{\pert \sim p(\pert | \scal)} \left[\text{tr}[\mbf{B}^{\top} \hwr \scalm] \right] &= \mathbb{E}_{\pert \sim p(\pert | \scal)} \left[ \text{tr}[\scalm \mbf{B}^{\top} \hwr] \right] \\
    &= \mathbb{E}_{\pert \sim p(\pert | \scal)} \left[\text{tr}\left[\scalm \scalm^{-1} \pert \frac{\partial \mathcal{L}_T}{\partial \w} (\hyper_0 + \pert, \w_0) \hwr \right]\right] \\
    &= \mathbb{E}_{\pert \sim p(\pert | \scal)} \left[ \text{tr}\left[ \frac{\partial \mathcal{L}_T}{\partial \w} (\hyper_0 + \pert, \w_0) \hwr  \pert \right] \right]\\
    &= \mathbb{E}_{\pert \sim p(\pert | \scal)} \left[\text{tr}\left[ \frac{\partial \mathcal{L}_T}{\partial \w} (\hyper_0 + \pert, \w_0) \Delta \w \right] \right]\\
    &= \mathbb{E}_{\pert \sim p(\pert | \scal)} \left[\text{tr}\left[\frac{\partial \mathcal{L}_T}{\partial \mbf{y}} (\hyper_0 + \pert, \w_0) \Delta \mbf{y} \right] \right],
\end{align}
where $\Delta \w = \hwr \pert$ and $\Delta \mbf{y} = \mbf{J}_{\mbf{y} \mbf{w}} \Delta \w$. On the other hand, the first term is:
\begin{align}
    \mathbb{E}_{\pert \sim p(\pert | \scal)} [ \text{tr} [ \mbf{A} (\hyper_0, \w_0, \pert) &\hwr \scalm \hwr^{\top} ] ] \\
    &= \mathbb{E}_{\pert \sim p(\pert | \scal)} \left[ \text{tr} \left[\frac{\partial^2 \mathcal{L}_T}{\partial \w^2} (\hyper_0 + \pert, \w_0) \hwr \scalm \hwr^{\top} \right] \right] \\
    \label{app:indep}
    &\approx \mathbb{E}_{\pert \sim p(\pert | \scal)} \left[ \text{tr} \left[\frac{\partial^2 \mathcal{L}_T}{\partial \w^2} (\hyper_0 + \pert, \w_0) \hwr \pert \pert^{\top} \hwr^{\top} \right] \right] \\
    &= \mathbb{E}_{\pert \sim p(\pert | \scal)} \left[ \text{tr} \left[\frac{\partial^2 \mathcal{L}_T}{\partial \w^2} (\hyper_0 + \pert, \w_0) \Delta \w \Delta \w^{\top} \right] \right]\\
    &= \mathbb{E}_{\pert \sim p(\pert | \scal)} \left[ \text{tr} \left[\Delta \w^{\top} \frac{\partial^2 \mathcal{L}_T}{\partial \w^2} (\hyper_0 + \pert, \w_0) \Delta \w  \right] \right]\\
    \label{app:gnm-approx}
    &\approx \mathbb{E}_{\pert \sim p(\pert | \scal)} \left[ \text{tr} \left[\Delta \w^{\top} \mbf{J}_{\mbf{y}\mbf{w}}^{\top} \frac{\partial^2 \mathcal{L}_T}{\partial \mbf{y}^2} (\hyper_0 + \pert, \w_0) \mbf{J}_{\mbf{y}\mbf{w}} \Delta \w  \right] \right] \\
    &= \mathbb{E}_{\pert \sim p(\pert | \scal)} \left[ \text{tr} \left[\Delta \mbf{y}^{\top} \frac{\partial^2 \mathcal{L}_T}{\partial \mbf{y}^2} (\hyper_0 + \pert, \w_0) \Delta \mbf{y}  \right] \right] 
\end{align}
The third line (Eqn.~\ref{app:indep}) assumes that $\sfrac{\partial^2 \mathcal{L}_T}{\partial \w^2} (\hyper_0 + \pert, \w_0)$ and $\pert$ are independent, and Eqn.~\ref{app:gnm-approx} is a Gauss-Newton approximation. Therefore, first and second terms correspond to the first- and second-order Taylor approximations to the loss. For $\Delta$-STNs, we linearized the predictions with respect to the loss. This can be explained by the fact that the loss functions such as mean squared error and cross entropy are locally quadratic and closely matches the second order approximation.

\section{Structured Hypernetwork Representation for Convolutional Layers}
\label{append:conv-nn}
In this section, we describe a structured best-response approximation for convolutional layers. Considering $i$-th layer of a convolutional neural network, let $C_i$ denote number of filters and $K_i$ denote size of the kernel. Let $\mbf{W}^{(i, c)} \in \mathbb{R}^{C_{i - 1} \times K_i \times K_i}$ and $\mbf{b}^{(i, c)} \in \mathbb{R}$ denote weights and bias at $c$-th convolutional kernel, where $c \in \{1, ..., C_i\}$. We propose to approximate the layer-wise best-response function as follows:
\begin{equation}
\begin{gathered}
    \mbf{W}_{\hw}^{(i, c)} (\hyper, \hyper_0)  = \mbf{W}_{\text{general}}^{(i, c)} + \left((\hyper - \hyper_0)^{\top} \mbf{u}^{(i, c)} \right) \odot \mbf{W}_{\text{response}}^{(i, c)}\\
    \mbf{b}_{\hw}^{(i, c)} (\hyper, \hyper_0) = \mbf{b}_{\text{general}}^{(i, c)} + \left((\hyper - \hyper_0)^{\top}\mbf{v}^{(i, c)}  \right) \odot \mbf{b}_{\text{response}}^{(i, c)},
\end{gathered}
\end{equation}
where $\mbf{u}^{(i, c)}, \mbf{v}^{(i, c)} \in \mathbb{R}^h$. Observe that these formulas are linear in $\hyper$ similar to the approximation for fully-connected layers and analogous to that of the original STN. This architecture is also memory efficient and tractable to compute, and allows parallelism. The approximation requires $2|\mbf{W}^{(i, c)}| + h$ and $2|\mbf{b}^{(i, c)}| + h$ parameters to represent the weight and bias, and two additional element-wise multiplications in the forward pass. Summing over all channels, the total number of parameters is $2p + 2 h C_i$ for each layer, where $p$ is the number parameters for the ordinary CNN layer. Thus, the $\Delta$-STN incurs little memory overhead compared to training an ordinary CNN.

\section{Fixed Nonlinear Function on Hyperparameters}
\label{app:restricted}
In general, the hyperparameters have restricted domains. For example, the weight decay has to be a positive real number and the dropout rate has to be in between 0 and 1. Hence, we apply a fixed non-linear function $s \colon \mathbb{R}^h \to \mathbb{R}^h$ on the hyperparameters to ensure that hyperparameters are in its domain and optimize the hyperparameters on an unrestricted domain. Fixing $\hyper_0 \in \mathbb{R}^h$, the training objective for the hypernetwork with a fixed nonlinear transformation on the hyperparameters is as follows:
\begin{align}
    \min_{\mbi{\Theta} \in \mathbb{R}^{m \times h}} \mathbb{E}_{\pert \sim p(\pert|\scal)} \left[ \mathcal{L}_T (s(\hyper_0 + \pert), \res_{\hw} (\hyper_0 + \pert, \hyper_0))\right]
\end{align}
We also use such transformation to restrict the hyperparameters to be in its search space.

\section{Example of STN's Training Objective having Incorrect Fixed Point}
\label{sec:linear-reg-example}
Consider a linear regression with $L_2$ regularization where the training objective is defined as:
\begin{align}
    \mathcal{L}_T (\lambda, \w) &= \frac{1}{2n} \norm{\mbf{X} \w - \mbf{t}}^2 + \frac{\lambda}{2 n} \norm{\w}^2,
\end{align}
where $\mbf{X} \in \mathbb{R}^{n \times m}$ and $\mbf{t} \in \mathbb{R}^n$ are the input matrix and target vector, respectively. Given $\lambda_0 \in \mathbb{R}$, under the STN's training objective (Eqn.~\ref{eq:lor-objective}), we aim to minimize:
\begin{align}
    \min_{\w_0 \in \mathbb{R}^m} \mathbb{E}_{\epsilon \sim p(\epsilon|\sigma)} \left[\frac{1}{2n} \norm{\mbf{X} (\w_0 + \mbi{\Theta}\epsilon) - \mbf{t}}^2 + \frac{\lambda_0 + \epsilon}{2 n} \norm{\w_0 + \mbi{\Theta}\epsilon}^2 \right]
\end{align}
Simplifying the above equation, we get:
\begin{align}
    \mathbb{E}_{\pert \sim p(\pert | \scal)} \left[\text{\circled{1}} +\text{\circled{2}} \right],
\end{align}
where each component is:
\begin{align}
    \mathbb{E}_{\epsilon \sim p(\epsilon, \sigma)} \left[\text{\circled{1}}\right] &= \frac{1}{2n} \left(\w_0^{\top} \mbf{X}^{\top} \mbf{X} \w_0 + \hwr^{\top} \mbf{X}^{\top} \mbf{X} \hwr \sigma^2 + \mbf{t}^{\top} \mbf{t} - 2 \w_0^{\top} \mbf{X}^{\top} \mbf{t} \right) \\
     \mathbb{E}_{\epsilon \sim p(\epsilon, \sigma)} \left[\text{\circled{2}}\right] &= \frac{1}{2n} \left(\w_0^{\top} \w_0 \lambda_0 + 2 \w_0^{\top} \hwr \sigma^2 + \hwr^{\top} \hwr \lambda_0 \sigma^2 \right)
\end{align}
The gradient with respect to $\w_0$ is the following:
\begin{align}
    \frac{\partial}{\partial \w_0} \mathbb{E}_{\epsilon \sim p(\epsilon|\sigma)} \left[ \mathcal{L}_T (\lambda_0 + \epsilon, \res_{\hw} (\lambda_0 + \epsilon, \lambda_0)) \right] =  \frac{1}{n} \left( \mbf{X}^{\top} \mbf{X} \w_0 - \mbf{X}^{\top} \mbf{t} + \w_0 \lambda_0 + \hwr \sigma^2 \right)
\end{align}
Setting the above equation equal to $\mbf{0}$, the optimal solution $\w_0^*$ under STN's training objective is as follows:
\begin{align}
    \w_0^* = (\mbf{X}^{\top} \mbf{X} + \lambda_0 \mbf{I})^{-1} (\mbf{X}^{\top} \mbf{t} - \hwr \sigma^2)
    \label{eq:linear-reg-incorrect}
\end{align}
and  the optimal $\hwr^*$ is:
\begin{align}
    \hwr^* = - (\mbf{X}^{\top} \mbf{X} + \lambda_0 \mbf{I})^{-1} \mbf{w}_0^*,
    \label{eq:linear-reg-l2-jacob}
\end{align}
Comparing Eqn.~\ref{eq:linear-reg-incorrect} to the optimal weight $\w^*$ of linear regression with $L_2$ regularization,
\begin{align}
    \w^* = (\mbf{X}^{\top} \mbf{X} + \lambda_0 \mbf{I})^{-1} \mbf{X}^{\top} \mbf{t},
\end{align}
the optimal solution for $\w_0^*$ under STN's training objective is incorrect, as $\sigma^2 > 0$. Moreover, the inaccuracy of $\w_0^*$ also affects accuracy of the best-response Jacobian as shown in Eqn.~\ref{eq:linear-reg-l2-jacob}. In contrast, the proposed objective in Eqn.~\ref{eq:new-objective_hw0} recovers the correct solution for both the weight $\w_0^*$ and best-response Jacobian $\hwr^*$.

\section{Experiment Details}
In this section, we present additional details for each experiment.

\subsection{Toy Experiments}
\label{exp-toy}
For toy experiments, the datasets were randomly split into training and validation set with ratio 80\% and 20\%. We fixed the perturbation scale to 1, and used $T_{\text{train}} = 10$ and $T_{\text{valid}} = 1$ for all experiments. We further used 100 iterations of warm-up that does not update the hyperparameters. Note that the warm-up still perturbs the hyperparameters. The batch size was 10 for datasets with less than 1000 data points, and 100 for others. In training, we normalized the input features and targets to have a zero mean and unit variance. The training objective for the ridge regression was as follows:
\begin{align}
 \mathcal{L}_T (\lambda, \w) = \frac{1}{2n} \norm{\mbf{X} \mbf{w} - \mbf{t}}^2 + \frac{\exp(\lambda)}{2n} \norm{\w}^2,
 \label{eq:linear-reg-obj}
\end{align}
where $n$ is the number of training data, $\mbf{X} \in \mathbb{R}^{n \times m}$ is the input matrix and $\mbf{t} \in \mathbb{R}^n$ is the target. We initialized the regularization penalty $\lambda$ to $1$ for datasets with less than 1000 data points and 50 for others. The hyperparameter updates for linear regression with $L_2$ regularization on UCI datasets is shown in figure~\ref{fig:comparison_uci_linear_reg}.

\begin{figure}[h!]
\centering
 \includegraphics[width=0.30\textwidth]{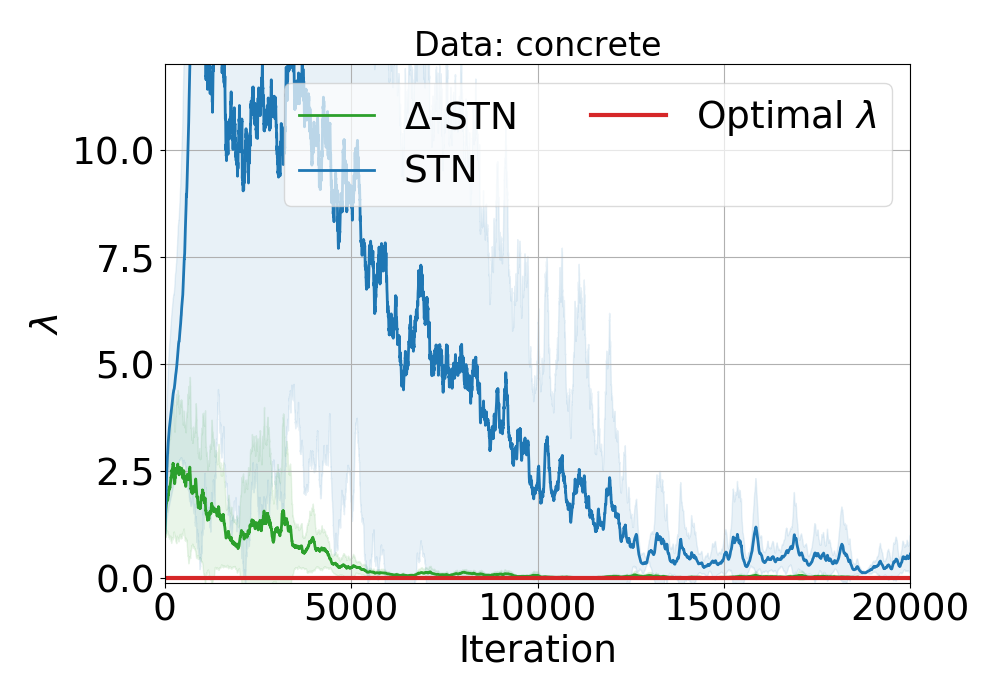}
 \includegraphics[width=0.30\textwidth]{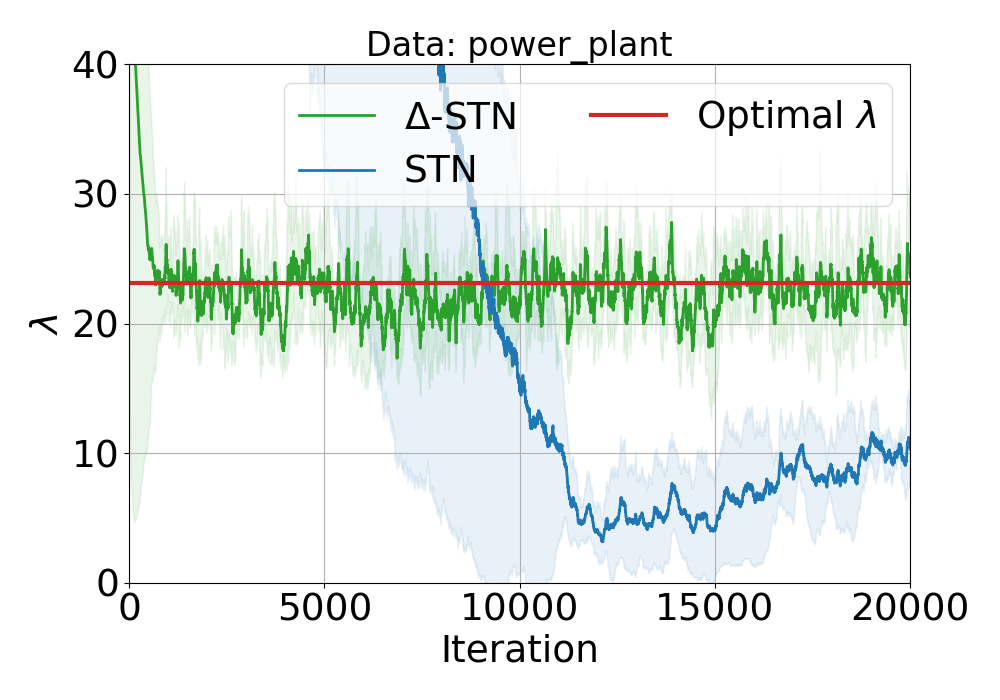}
 \includegraphics[width=0.30\textwidth]{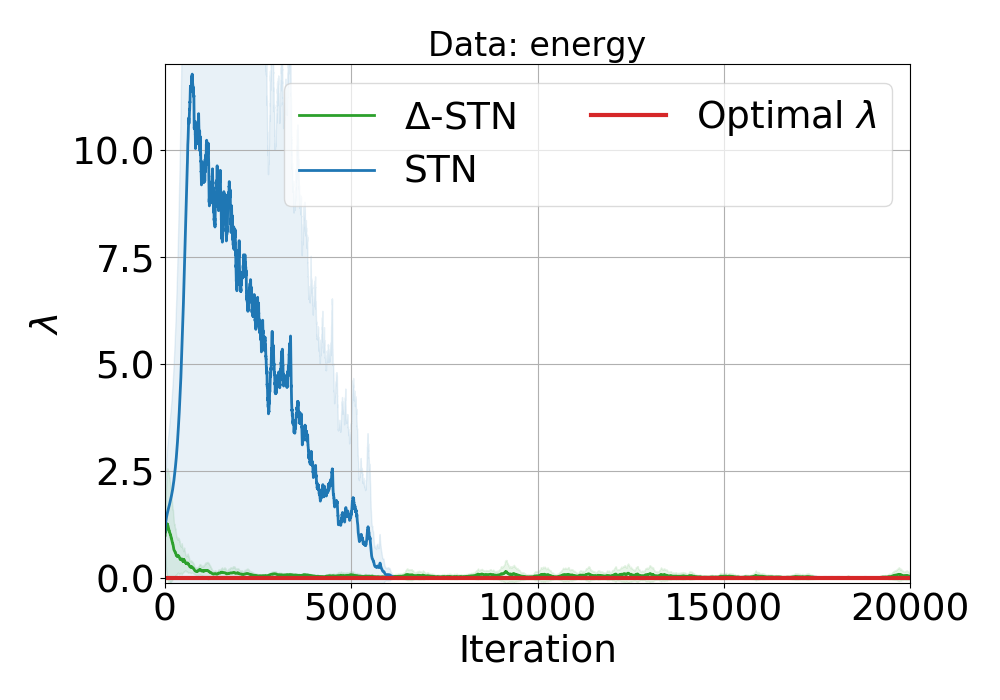}\\
 \includegraphics[width=0.30\textwidth]{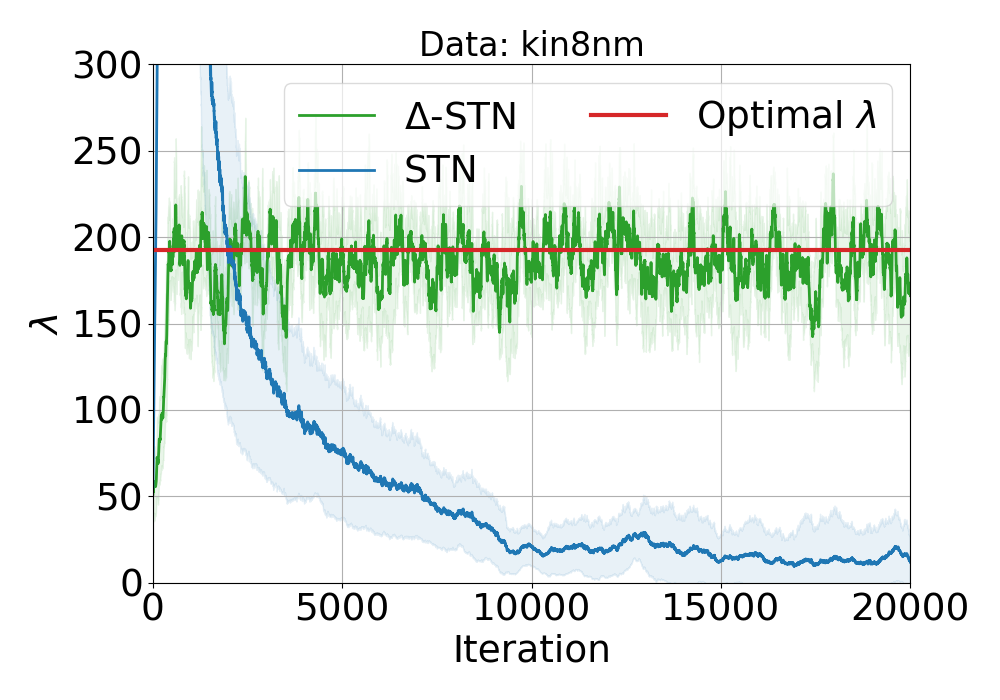}
 \includegraphics[width=0.30\textwidth]{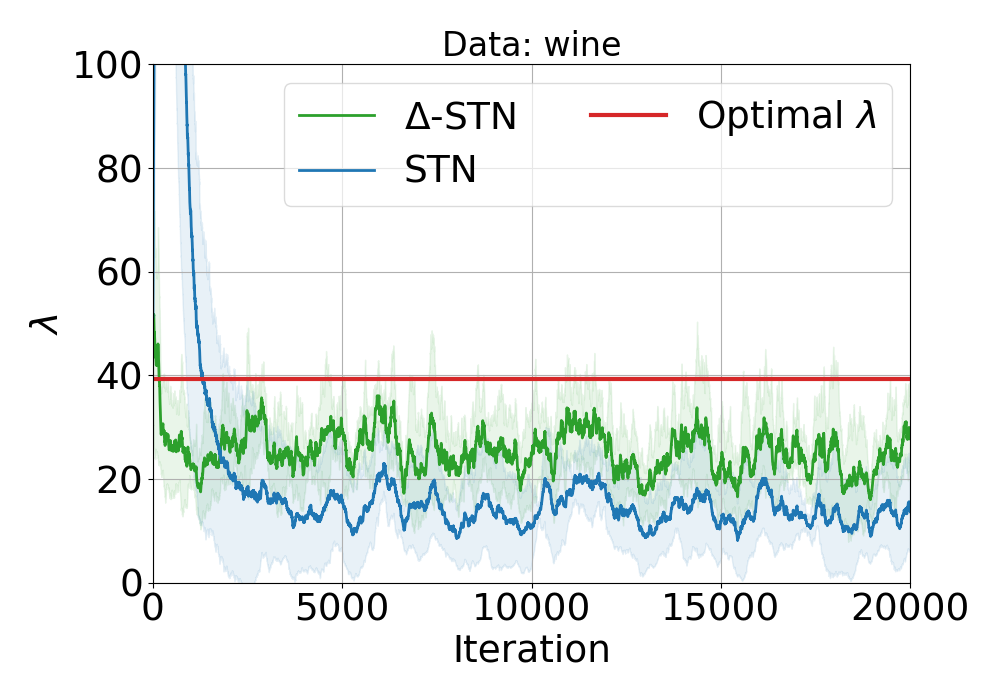}
\includegraphics[width=0.30\textwidth]{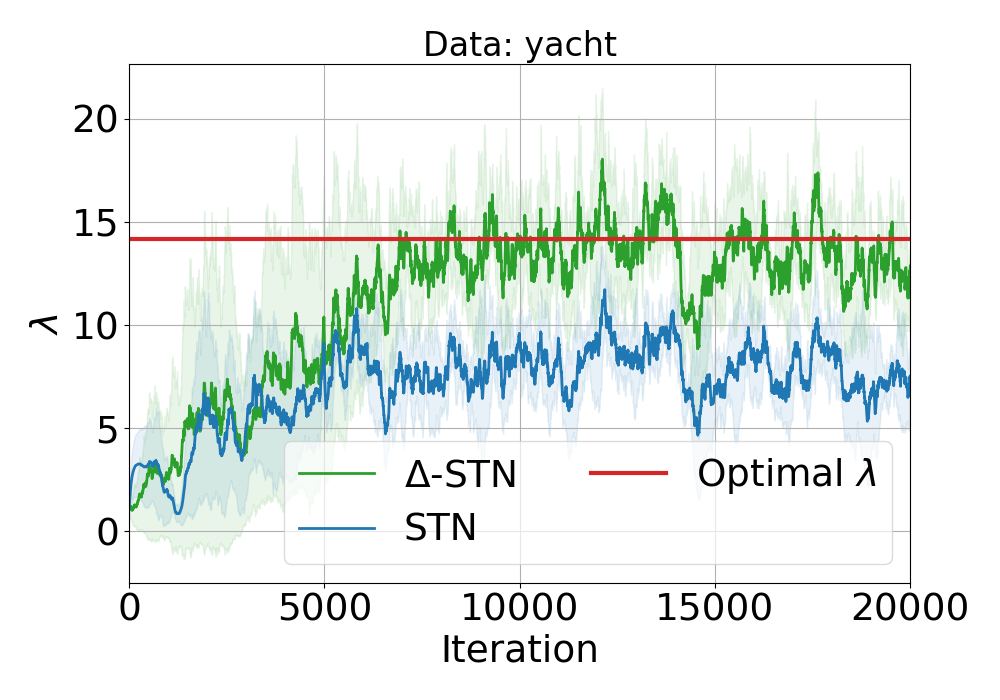}
 \caption{A comparison of STNs and $\Delta$-STNs on linear regression with $L_2$ regularization.}
 \label{fig:comparison_uci_linear_reg}
\end{figure}

We note that dividing the regularization penalty by the number of training set as in Eqn.~\ref{eq:linear-reg-obj} is non-standard. Because the STN is sensitive to the scale of the hyperparameters, we also experimented with the objective that does not divide the regularization penalty by the number of training data. As shown in figure \ref{fig:comparison_uci_linear_reg2}, $\Delta$-STNs still outperform STNs in terms of convergence, accuracy, and stability.
\begin{figure}[h!]
\centering
 \includegraphics[width=0.30\textwidth]{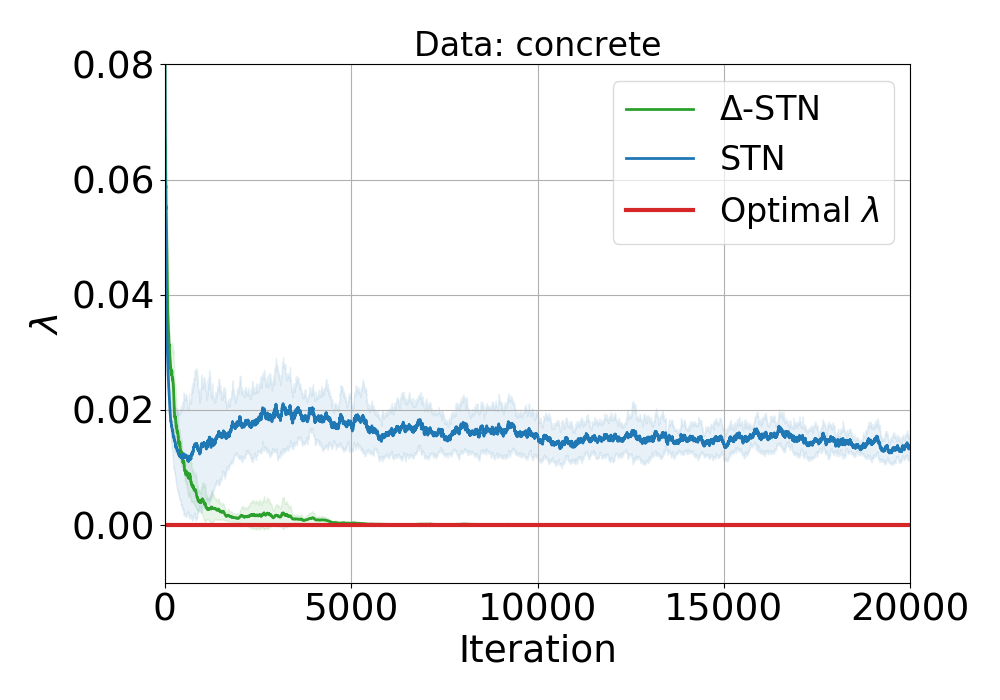}
 \includegraphics[width=0.30\textwidth]{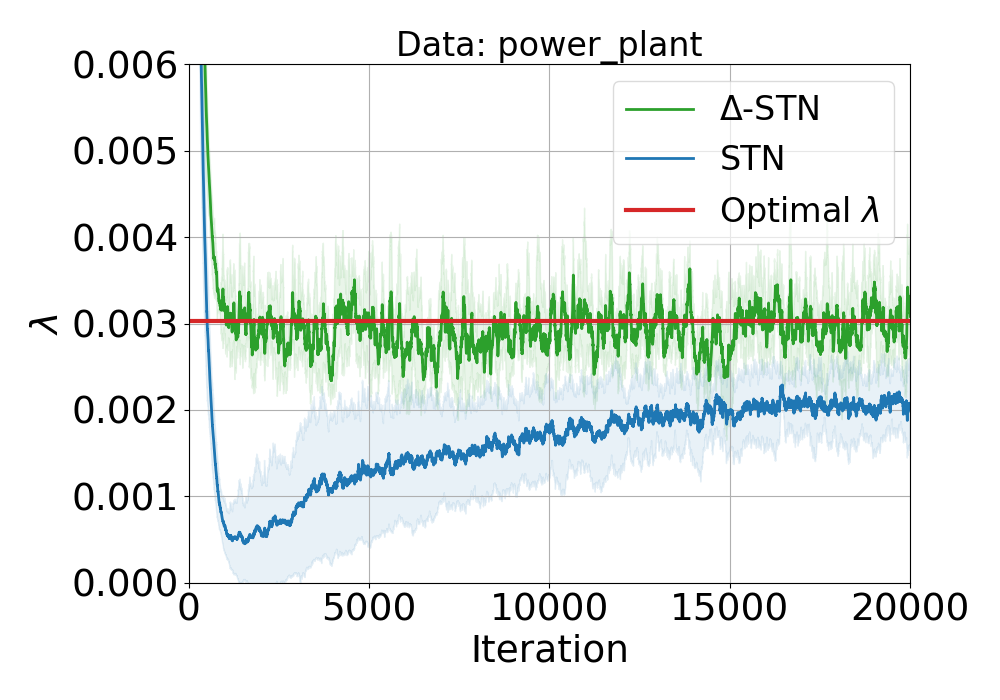}
 \includegraphics[width=0.30\textwidth]{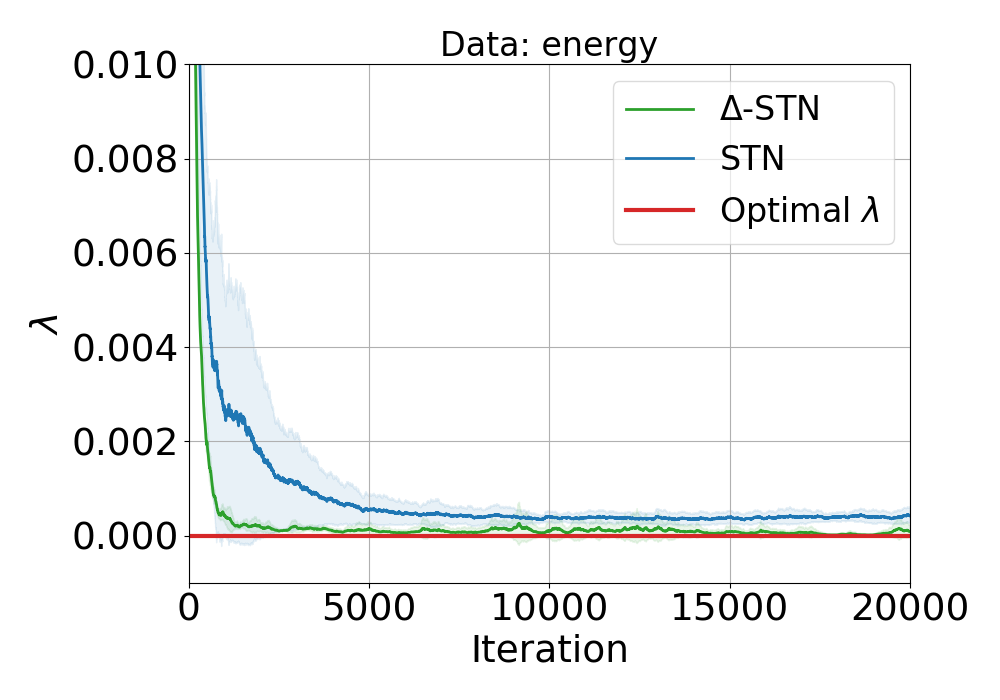}\\
 \includegraphics[width=0.30\textwidth]{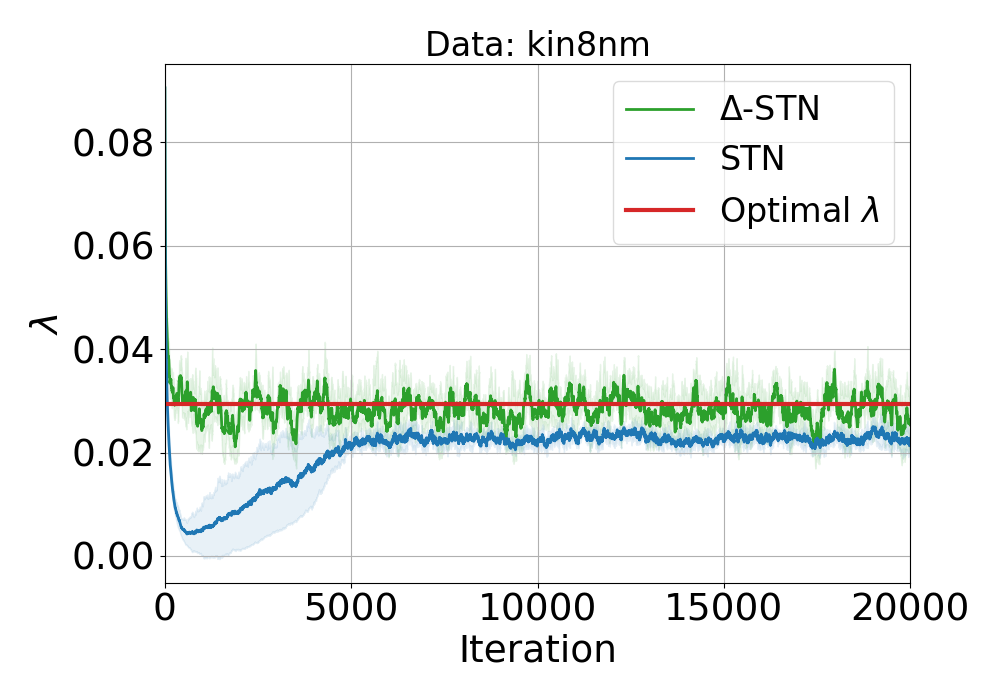}
 \includegraphics[width=0.30\textwidth]{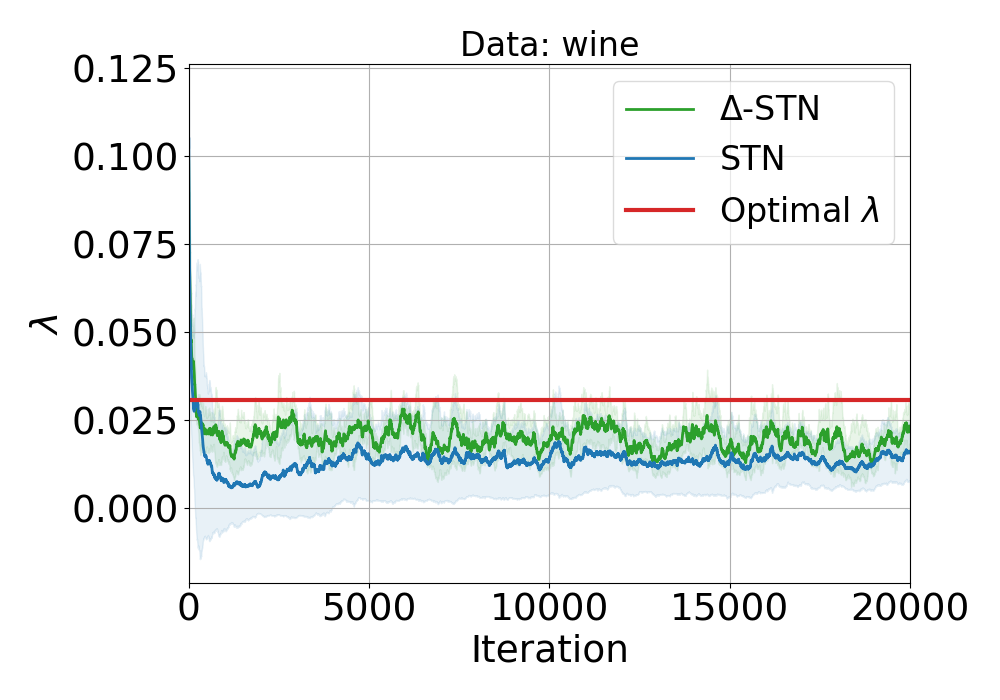}
\includegraphics[width=0.30\textwidth]{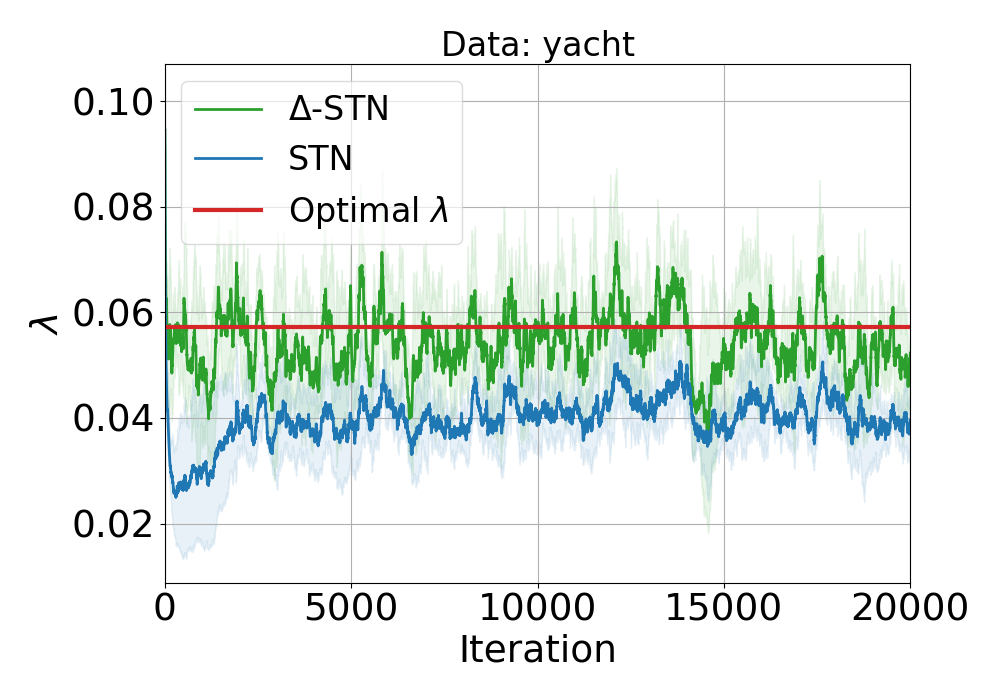}
 \caption{A comparison of STNs and $\Delta$-STNs on linear regression with $L_2$ regularization without dividing the regularization penalty by the number of training data.}
\label{fig:comparison_uci_linear_reg2}
\end{figure}

\begin{wrapfigure}[10]{R}{0.55\textwidth}
    \centering
    \vspace{-0.1cm}
    \includegraphics[width=0.55\textwidth]{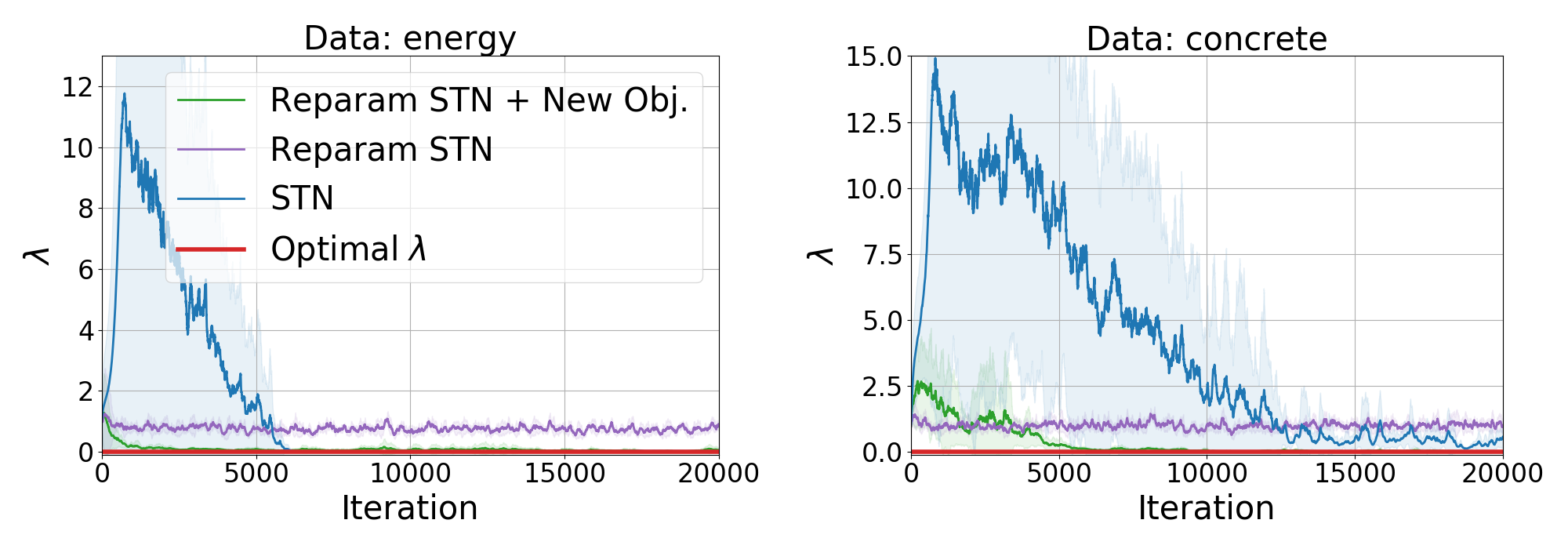}
    \vspace{-0.5cm}
    \caption{A comparison of hyperparameter updates found by STNs and $\Delta$-STNs, decoupling effects of reparameterization and modified objective.}
    \label{fig:decouple}
\end{wrapfigure}
In figure~\ref{fig:decouple}, we decoupled the centered parameterization from all other changes to the training objective for datasets that experimented spikes in the hyperparameter updates. The centering (``Reparam STN'') described in section~\ref{sec:method-centered-param} by itself eliminated the spike in the hyperparameter updates as our analysis predicts. The modification of the training objective (``Reparam STN + New Obj.'') described in section~\ref{sec:method-modified-objective} further helped improving the accuracy of the fixed point as detailed in appendix~\ref{sec:linear-reg-example}.

Similarly, the training objective for dropout experiments was:
\begin{align}
  \mathcal{L}_T (\lambda, \w) = \frac{1}{2n} \norm{ (\mbf{R} \odot \mbf{X}) \mbf{w} - \mbf{t}}^2,
\end{align}
where $\mbf{R} \in \mathbb{R}^{n \times m}$ is a random matrix with $\mbf{R}_{ij} \sim \text{Bern}(\lambda)$ and $\odot$ is an element-wise multiplication. Note that linear regression with input dropout is, in expectation, equivalent to ridge regression with a particular form of regularization penalty~\citep{srivastava2014dropout}. We initialized the dropout rate to 0.2 for all datasets. For both experiments, the validation objective was the following:
\begin{align}
 \mathcal{L}_V (\lambda, \w) = \frac{1}{2 v} \norm{\mbf{X_{\text{valid}}} \mbf{w} - \mbf{t}_{\text{valid}}}^2,
\end{align}
where $v$ is the number of validation data, and $\mbf{X}_{\text{valid}} \in \mathbb{R}^{v \times m}$ and $\mbf{t}_{\text{valid}} \in \mathbb{R}^v$ are the inputs and targets of the validation data.

We adopted the same experimental settings for linear networks with Jacobian norm regularization. The model was consisted of 5 hidden layers and each hidden layer had a square matrix with the same dimension as the input. The training objective was as follows:
\begin{align}
    \mathcal{L}_T (\lambda, \w) = \frac{1}{2n} \norm{y(\mbf{X}; \w) - \mbf{t}}^2 + \frac{\exp(\lambda)}{2n} \norm{\frac{\partial y}{\partial \mbf{X}}(\mbf{X}; \w)}^2 ,
\end{align}
where $y$ is the linear network's prediction. The plots presented in figures \ref{fig:linear-regression-result} and \ref{fig:linear-network} show the updates of hyperparameters averaged over 5 runs with different random seeds.

\subsection{Image Classification}
\label{exp-image-class}

For the baselines (grid search, random search, and Bayesian optimization), the search spaces were as follows: dropout rates were in $[0, 0.85]$; the number of Cutout holes was in $[0, 4]$; Cutout length was in $[0, 24]$, noises added to contrast, brightness, and saturation were in $[0, 1]$; noise added to hue was in $[0, 0.5]$; the random scale and translation of an image were in $[0, 0.5]$; the random degree and shear were in $[0, 45]$. All other settings (e.g.~learning rate schedule, batch size) were the same as those of $\Delta$-STNs. We used Tune~\citep{liaw2018tune} for gird and random searches, and Spearmint~\citep{snoek2012practical} for Bayesian optimization.

For all experiments, we set training and validation update intervals to $T_{\text{train}} = 5$ and $T_{\text{valid}} = 1$, and the perturbation scale was initialized to 1. We further performed a grid search over the entropy weight from a set $\{1e^{-2}, 1e^{-3}, 1e^{-4}\}$ on 1 run and repeated the experiments with the selected entropy weight 5 times. We reported the average over 5 runs.

\subsubsection{MNIST}
\label{exp-detail-mnist}
We held out 15\% of the training data for validation. We trained a multilayer perceptron using SGD with a fixed learning rate 0.01 and momentum 0.9, and mini-batches of size 128 for 300 epochs. The MLP model consisted of 3 hidden layers with 1200 units and ReLU activations. We used 5 warm-up epochs that did not optimize the hyperparameters. In total of 3 dropout rates that control the inputs and per-layer activations were tuned. The hyperparameters were optimized using RMSProp with learning rate 0.01. We used entropy weights $\tau = 0.001$ for both STNs and $\Delta$-STNs, and initialized all dropout rates to 0.05, where the range for dropout was $[0, 0.95]$. We show the hyperparameter schedules found by $\Delta$-STNs on figure
\ref{fig:mnist-schedule}.

\subsubsection{FashionMNIST}
Similar to the MNIST experiment, we held out 15\% of the training data for validation. We trained Convolutional neural network using SGD with a fixed learning rate 0.01 and momentum 0.9, and mini-batches of size 128 for 300 epochs. The SimpleCNN model consisted of 2 convolutional layers with 16 and 32 filters, both with kernel size 5, followed by 2 fully-connected layers with 1568 hidden units and with ReLU activations on all hidden layers. We tuned 6 hyperparameters: (1) input dropout, (2) per-layer activation dropouts, and (3) Cutout holes and length. We set 5 warm-up epochs that did not optimize the hyperparameters and used RMSProp with learning rate 0.01 for optimizing the hyperparameters. The entropy weight was $\tau = 0.001$ for both STNs and $\Delta$-STNs. We initialized the dropout rates to $0.05$, the number of Cutout holes to $1$, and Cutout length to $4$. Dropout rates had a search space of $[0, 0.95]$, the number of Cutout holes had $[0, 4]$, and Cutout length had $[0, 24]$. The hyperparameter schedules prescribed by $\Delta$-STNs are shown in figure \ref{fig:mnist-schedule}.

\begin{figure}[t]
     \subfigure[MLP dropout]{\includegraphics[width=0.33\textwidth]{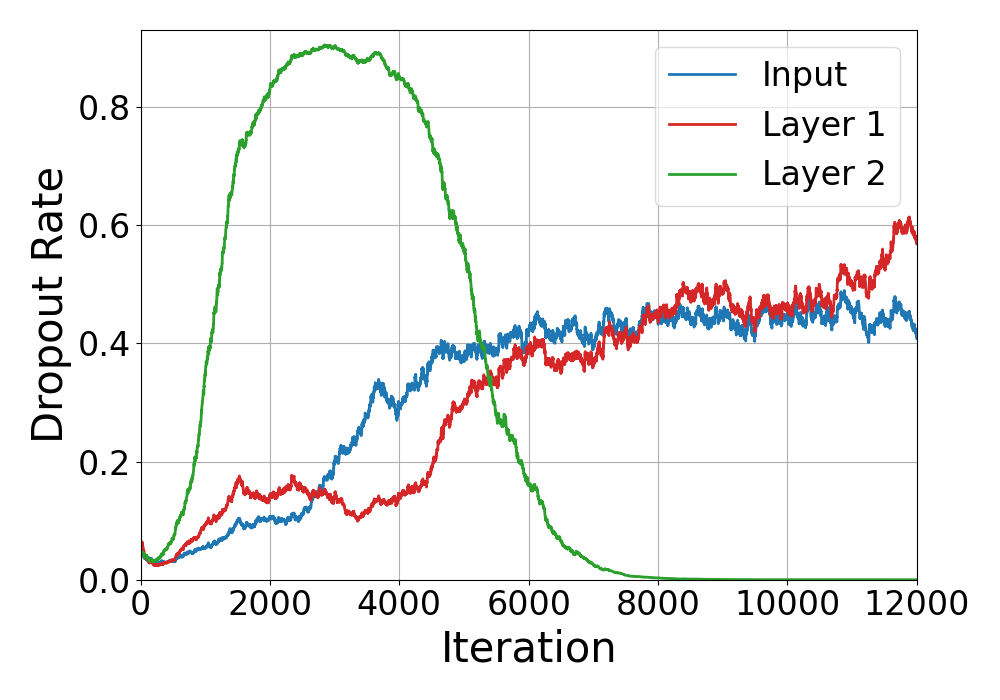}\label{fig:mnist-dropout}}
     \subfigure[SimpleCNN dropout]{\includegraphics[width=0.33\textwidth]{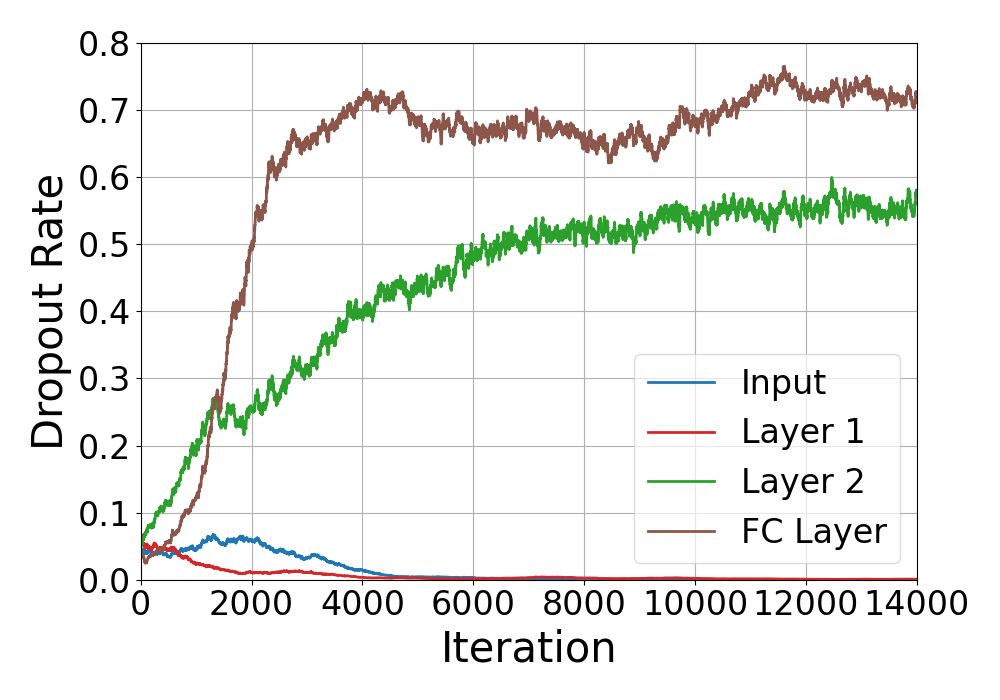}\label{fig:scnn-dropout}}
     \subfigure[SimpleCNN Cutout]{\includegraphics[width=0.33\textwidth]{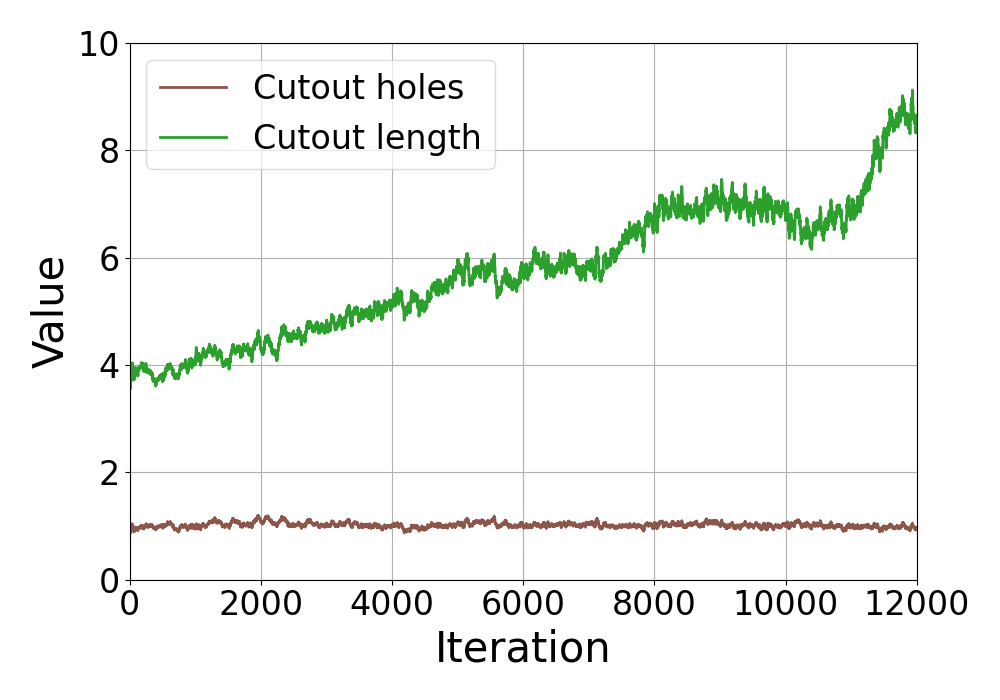}\label{fig:scnn-cutout}}
     \caption{Hyperparameter schedules prescribed by $\Delta$-STNs for \textbf{(a)} MNIST and \textbf{(b)}, \textbf{(c)} FashionMNIST datasets.}
     \label{fig:mnist-schedule}
     \vspace{-5mm}
 \end{figure}

\subsubsection{CIFAR10}
\begin{figure}[t]
     \subfigure[Schedule for dropouts]{\includegraphics[width=0.33\textwidth]{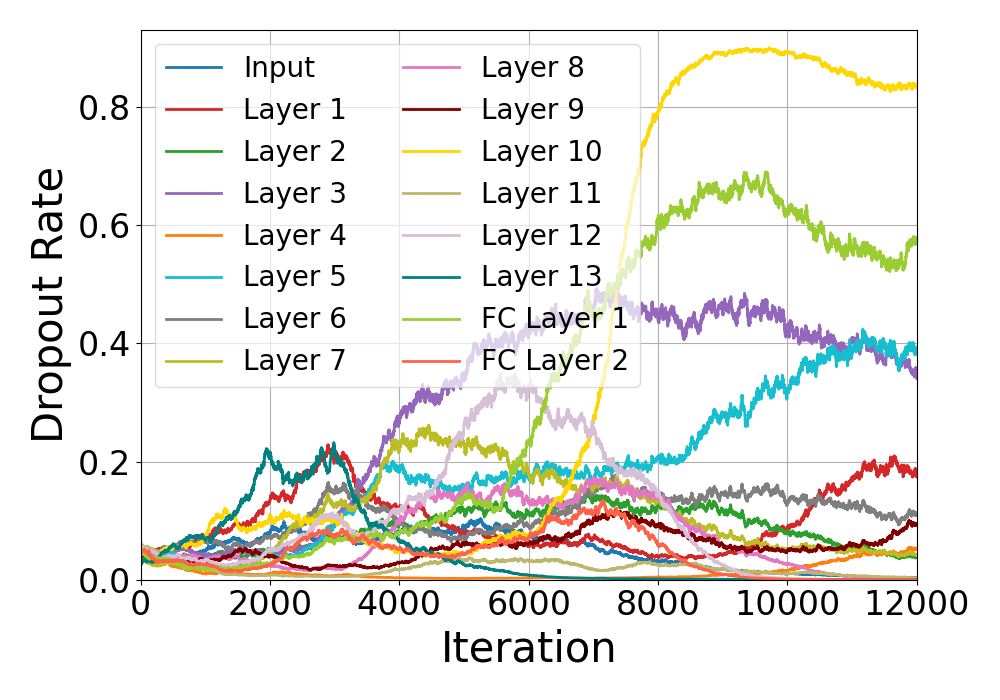}\label{fig:vgg-dropout}}
     \subfigure[Schedule for data augmentations]{\includegraphics[width=0.33\textwidth]{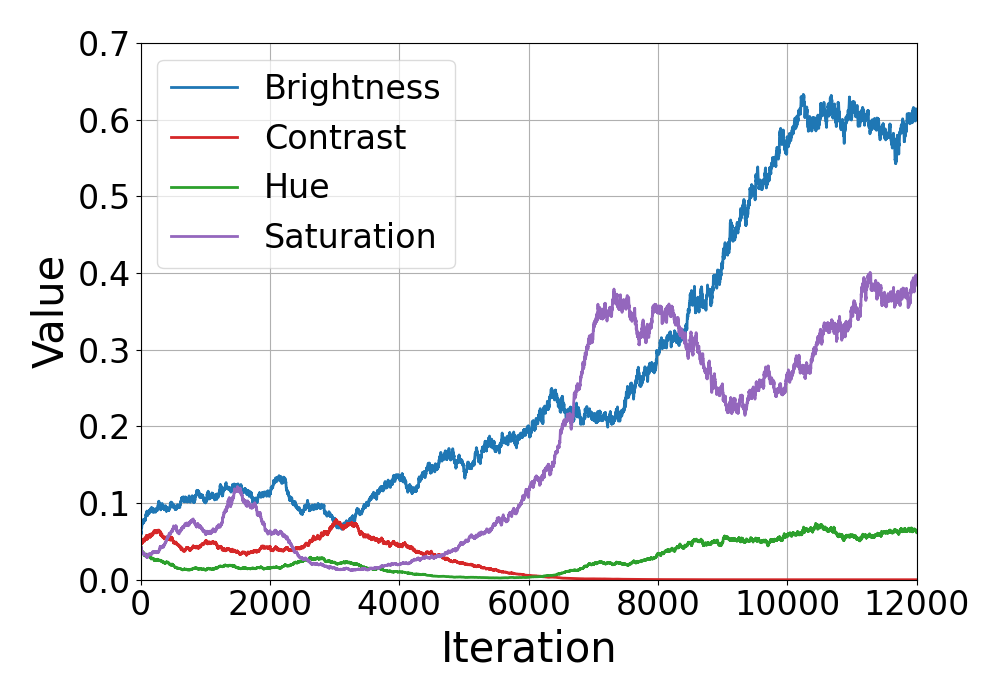}\label{fig:vgg-data-aug}}
     \subfigure[Schedule for Cutout]{\includegraphics[width=0.33\textwidth]{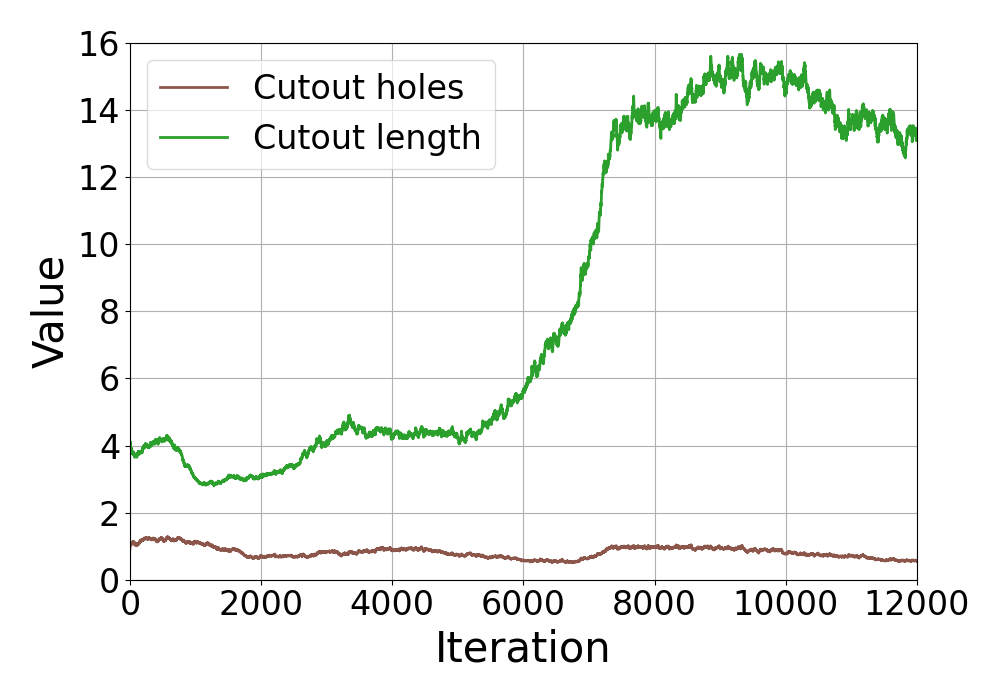}\label{fig:vgg-cutout}}
     \caption{Hyperparameter schedules found by $\Delta$-STNs on VGG16 for \textbf{(a)} dropout rates, \textbf{(b)} data augmentation parameters, and \textbf{(c)} Cutout parameters.}
     \label{fig:vgg}
     \vspace{-5mm}
\end{figure}
\begin{figure}[t]
     \subfigure[Schedule for dropouts]{\includegraphics[width=0.33\textwidth]{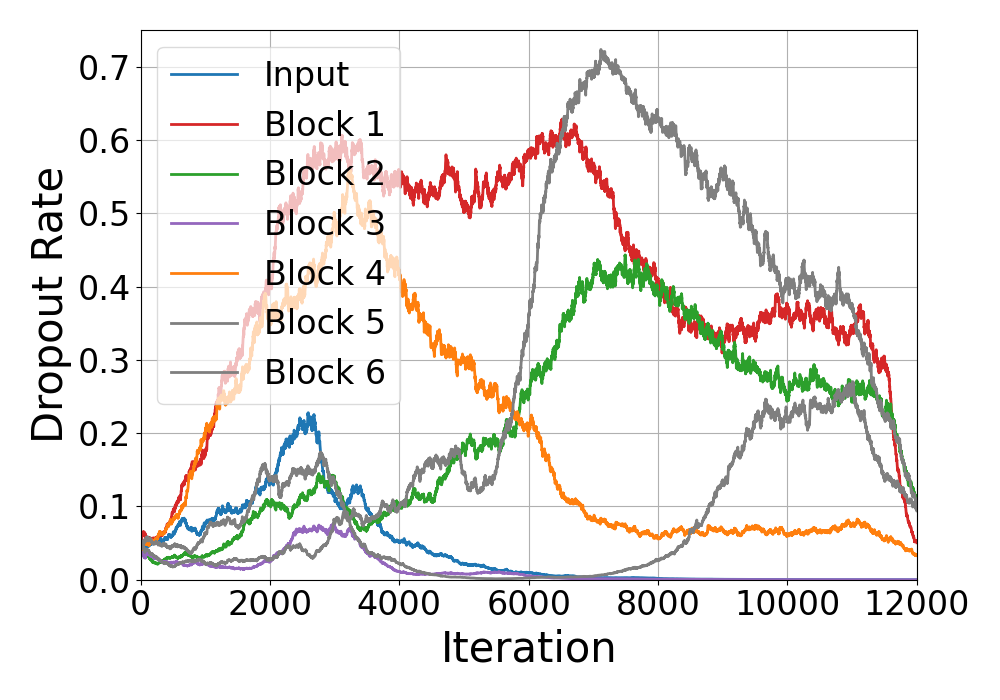}\label{fig:vgg-dropout}}
     \subfigure[Schedule for data augmentations]{\includegraphics[width=0.33\textwidth]{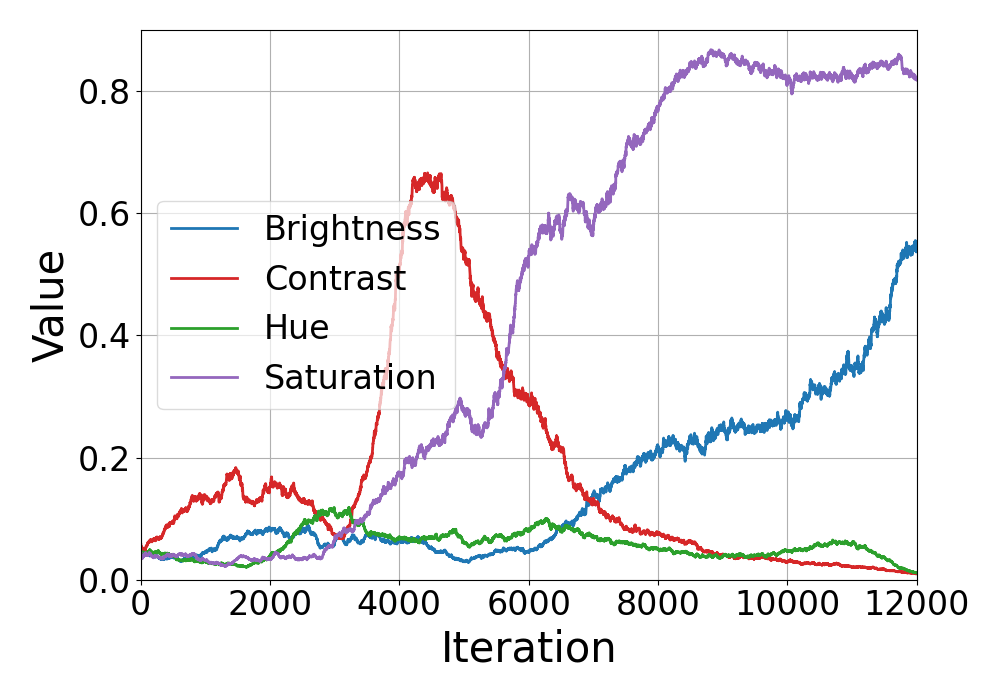}\label{fig:vgg-data-aug}}
     \subfigure[Schedule for Cutout]{\includegraphics[width=0.33\textwidth]{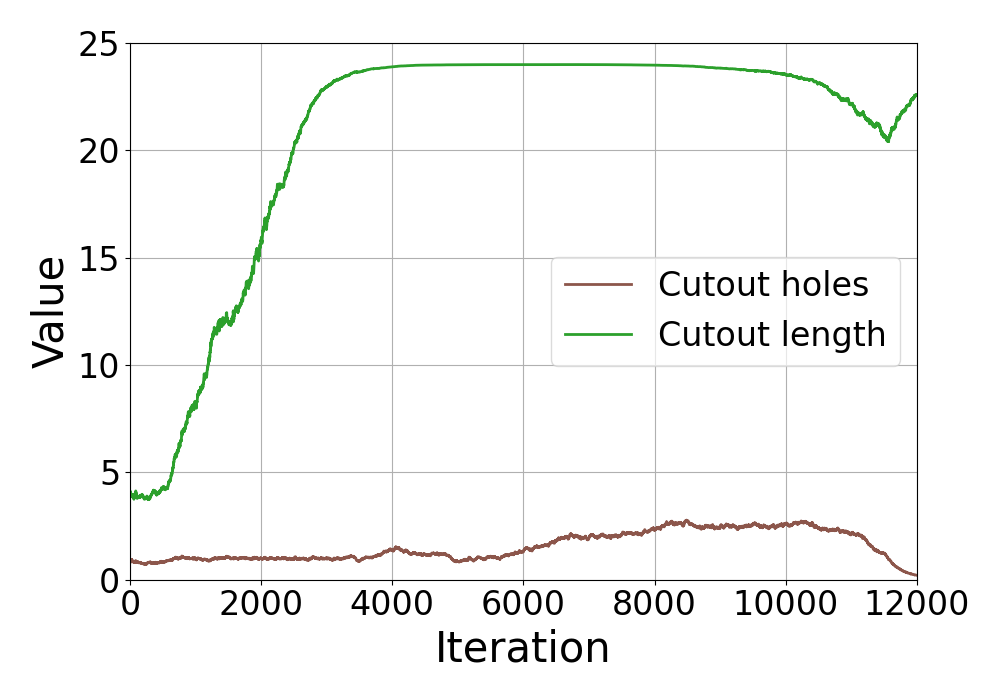}\label{fig:vgg-cutout}}
     \caption{Hyperparameter schedules found by $\Delta$-STNs on ResNet18 for \textbf{(a)} dropout rates, \textbf{(b)} data augmentation parameters, and \textbf{(c)} Cutout parameters.}
     \label{fig:resnet}
\end{figure}
We held out 20\% of the training data for validation. AlexNet, VGG16, and ResNet18 were trained with SGD with initial learning rates of 0.03, 0.01, 0.05 and momentum 0.9, using mini-batches of size 128. We decayed the learning rate after 60 epochs by a factor of 0.2 and trained the network for 200 epochs. We tuned 18 and 26 hyperparameters for AlexNet and VGG16: (1) input dropout, (2) per-layer activation dropouts, (3) scaling noise applied to the input, (4) Cutout holes and length, (5) amount of noise applied to hue, saturation, brightness, and contrast to the image, and (6) random translation, scale, rotation, and shear applied to input image. Similarly, 19 hyperparameters were optimized for ResNet18, where we applied dropout after each block and applied the same augmentation hyperparameters as AlexNet and VGG16. 

\begin{wraptable}[7]{L}{0.45\textwidth}
\small
\vspace{-0.4cm}
\caption{Final validation (test) accuracy of STN and $\Delta$-STN on image classification tasks.}
\vspace{-0.4cm}
\begin{center}
\begin{tabular}{|c|c|c|c|}
\hline
Network  & STN           & $\Delta$-STN  \\ \hline
AlexNet      & 83.96 (83.38) & \textbf{85.63 (85.19)} \\ \hline
VGG16       & 89.22 (88.66) & \textbf{90.96 (90.26)} \\ \hline
ResNet18     & 91.51 (90.16) & \textbf{93.46 (92.55)} \\ \hline
\end{tabular}
\end{center}
\label{tab:acc}
\end{wraptable}

The hyperparameters were optimized using RMSProp with a fixed learning rate 0.01. For all STN-type models, we used 5 epochs of warm-up for the model parameters. We used entropy weights of $\tau = 0.001$ on AlexNet, and $\tau = 0.0001$ and $\tau=0.001$ on VGG16 for $\Delta$-STNs and STNs, respectively. For ResNet18, we used entropy weight of $\tau = 0.0001$. The number Cutout holes was initialized to 1, the Cutout length was initialized to 4, and all other hyperparameters were initialized to 0.05. The search spaces for amount of noises added to contrast, brightness, saturation were $[0, 1]$ and that for noise added hue was $[0, 0.5]$. The search spaces for random translation and scale were also $[0, 0.5]$ and those for random rotation and shear were $[0, 45]$. The search spaces for all other hyperparameters were the same as those in experiments for FashionMNIST. We show the hyperparameter schedules obtained by $\Delta$-STNs on figure~\ref{fig:vgg} and figure~\ref{fig:resnet} for VGG16 and ResNet18 architectures. We also show the validation (training) accuracy obtained by each architecture in table~\ref{tab:acc}. $\Delta$-STNs showed a consistent improvement in validation accuracy compared to STNs.

\subsection{Language Modeling}
\label{exp:language}
\begin{wrapfigure}[9]{R}{0.5\textwidth}
    \centering
    \vspace{-0.5cm}
    \includegraphics[width=0.5\textwidth]{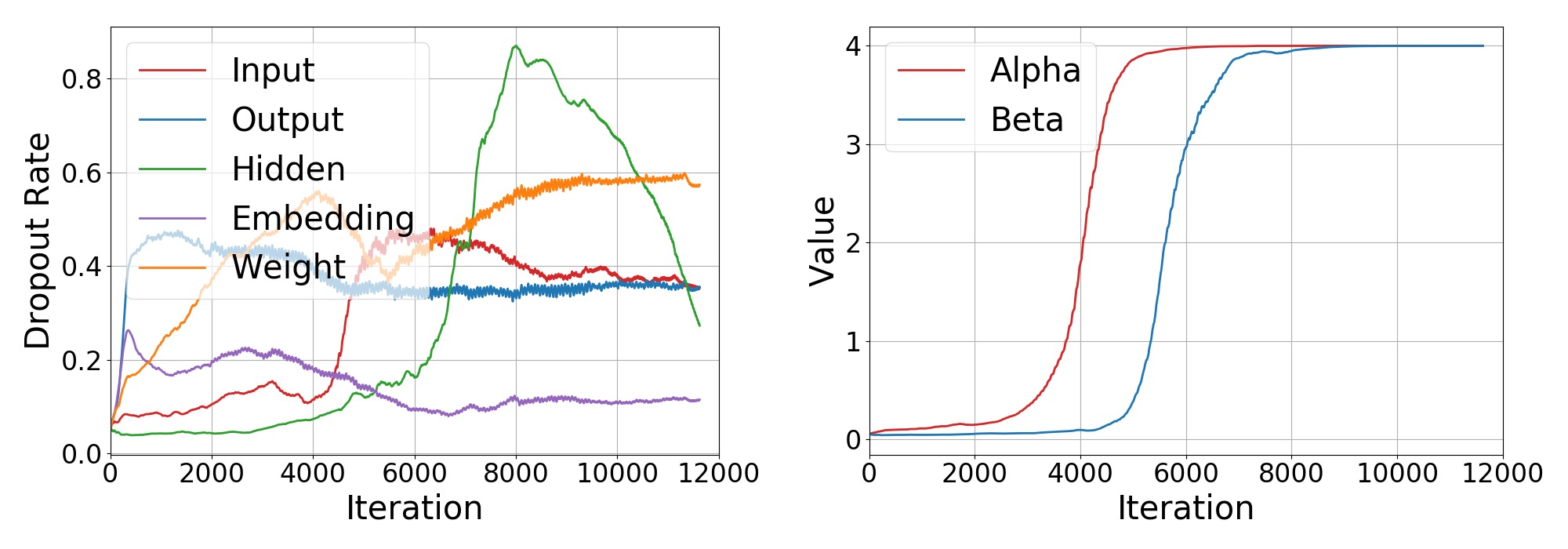}
    \vspace{-0.5cm}
    \caption{Hyperparameter schedules prescribed by $\Delta$-STNs on LSTM experiments.}
    \label{fig:lstm-result}
\end{wrapfigure}
We adopted the same experiment set-up to that of~\citet{mackay2019self}. We trained a 2 layer LSTM with 650 hidden units and 650-dimensional word embedding on sequences of length 70, using mini-batches of size 40. We used SGD with initial learning rate of 30 and decayed by a factor of 4 when the validation loss did not improve for 5 epochs. We further used gradient clipping with parameter $0.25$. The hyperparameters were optimized using Adam with a fixed learning rate of $0.01$. We used 10 warm-up epochs for both STNs and $\Delta$-STNs with a fixed perturbation scale of 1 and terminated the training when the learning rate for hypernetwork parameters decreased below $0.0003$. 

In total of 7 hyperparameters were optimized. We tuned variational dropout applied to the inputs, hidden states between layers, and the output to the model. Embedding dropout that sets an entire row of the word embedding matrix to $\mbf{0}$ was also tuned, eliminating certain words in the embedding matrix. We also regularized hidden-to-hidden weight matrix using DropConnect~\citep{wen2018flipout}. At last, activation regularization (AR) and temporal activation regularization (TAR) coefficients were tuned. For all RS, BO, STNs, and $\Delta$-STNs, the search spaces for AR and TAR were $[0,4]$ and we initialized them to 0.05. Similarly, all dropout rates were initialized to 0.05, and the search space was $[0, 0.95]$ for STNs and $\Delta$-STNs while it was $[0, 0.75]$ for RS and BO. We present the hyperparameter schedules found by $\Delta$-STNs in figure~\ref{fig:lstm-result}.
\begin{wraptable}[9]{R}{0.5\textwidth}
    \small
    \vspace{-3.3mm}
    \centering
    \setlength{\tabcolsep}{6pt}
    \begin{tabular}{@{}ccc@{}}
    \toprule
    \centering
    Method                  & Valid   & Test   \\ \midrule
    $p=0.5$, Fixed         & 0.059           & \textbf{0.054 }          \\
    $p = 0.5$ w/ Gaussian Noise     & 0.057          & 0.059            \\
    $p = 0.43$ (Final $\Delta$-STN Value)     & 0.059           & 0.060           \\
    STN         & 0.058  & 0.058   \\
    $\Delta$-STN             & \textbf{0.053}         &\textbf{ 0.054 }        \\ \bottomrule
    \end{tabular}
    \caption{A comparison of validation and test losses on MLP trained with fixed and perturbed input dropouts, and MLP trained with STNs and $\Delta$-STNs.}
    \label{tab:dropout-rate}
\end{wraptable}


\section{Additional Results}
\label{appendix:additional-results}
We present additional experiments in this section.

\subsection{Hyperparameter Schedules}

\begin{figure}[t]
    \centering
     \subfigure[Schedule found by $\Delta$-STNs]{\includegraphics[width=0.45\textwidth]{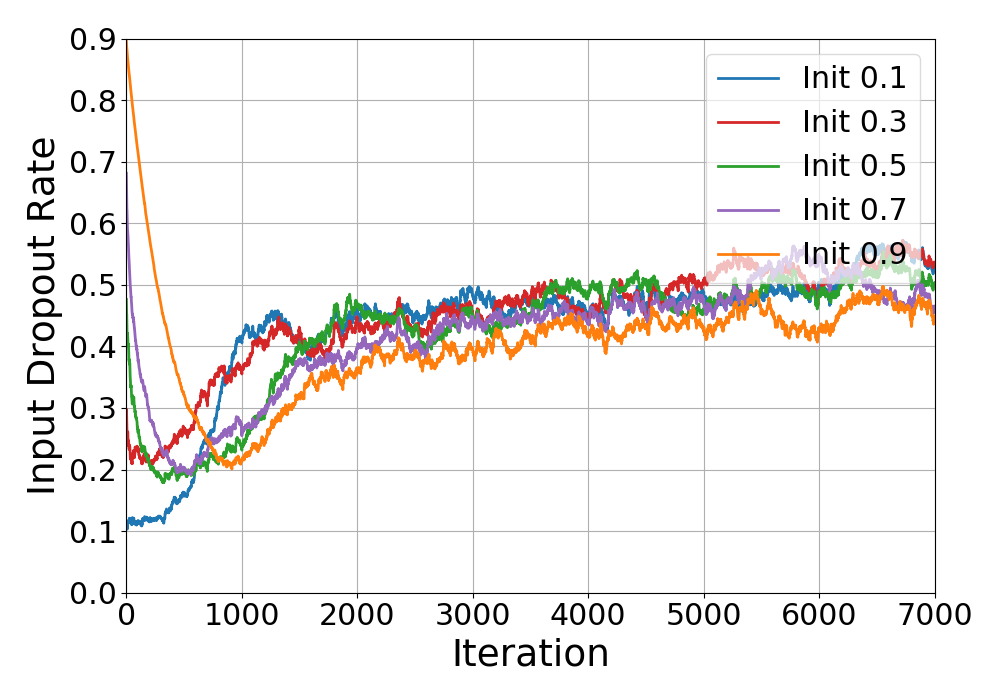}\label{fig:delta-schedule}}
     \subfigure[Schedule found by STNs]{\includegraphics[width=0.45\textwidth]{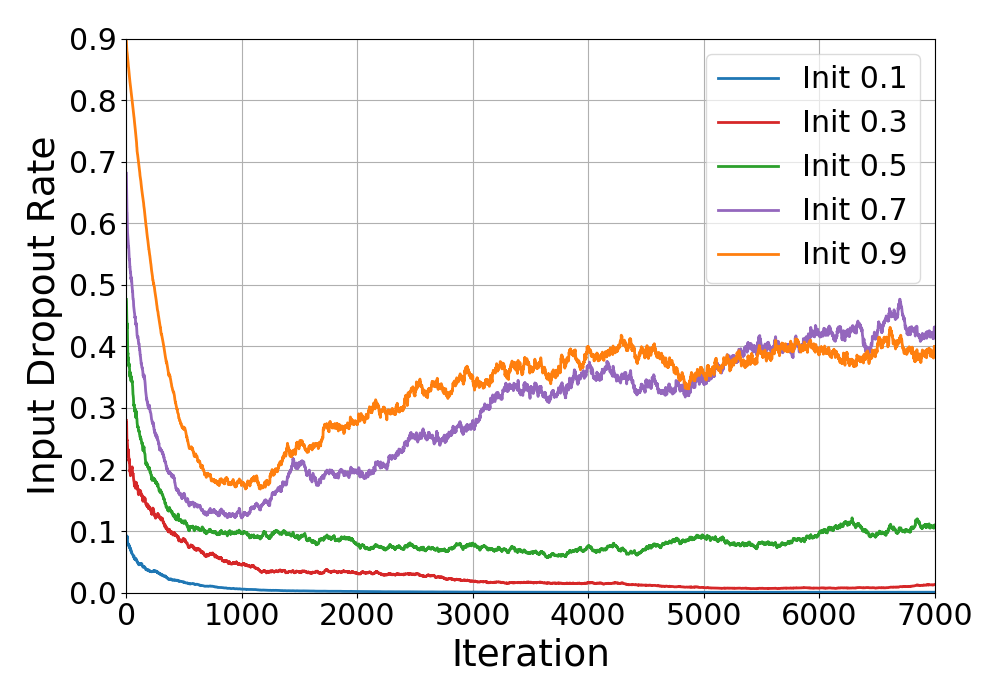}\label{fig:stn-schedule}}
     \caption{A comparison of input dropout schedules found by \textbf{(a)} $\Delta$-STNs and \textbf{(b)} STNs on MLP with different initialization. $\Delta$-STNs found the hyperparameter schedule more robustly and accurately compared to STNs.}
     \label{fig:mnist-schedule-init}
\end{figure}

Because the hyperparameters are tuned online, STNs do not use a fixed set of hyperparameters throughout the training. Instead, it finds hyperparameter schedules that outperforms fixed hyperparameters~\citep{mackay2019self}. We trained a multilayer perceptron on MNIST dataset and tuned the input dropout matrix with using STNs and $\Delta$-STNs. The same experimental configurations to those of MNIST experiment (appendix~\ref{exp-detail-mnist}) were used, except that we only tuned a single hyperparameter and fixed the perturbation scale to 1. We compared the hyperparameter schedules found by STNs and $\Delta$-STNs with different initializations. As shown in figure~\ref{fig:mnist-schedule-init}, $\Delta$-STNs found the hyperparameter schedule more robustly compared to STNs. We further compared best validation loss and corresponding test loss achieved by $\Delta$-STNs and STNs in table~\ref{tab:dropout-rate}. The fixed and perturbed dropout rates found by grid search and given by final $\Delta$-STN value were also compared. Our $\Delta$-STN was able to find the hyperparameter schedules more accurately and robustly as well.

\subsection{Sensitivity Studies}
\begin{figure}[t]
    \centering
     \subfigure[SimpleCNN]{\includegraphics[width=0.45\textwidth]{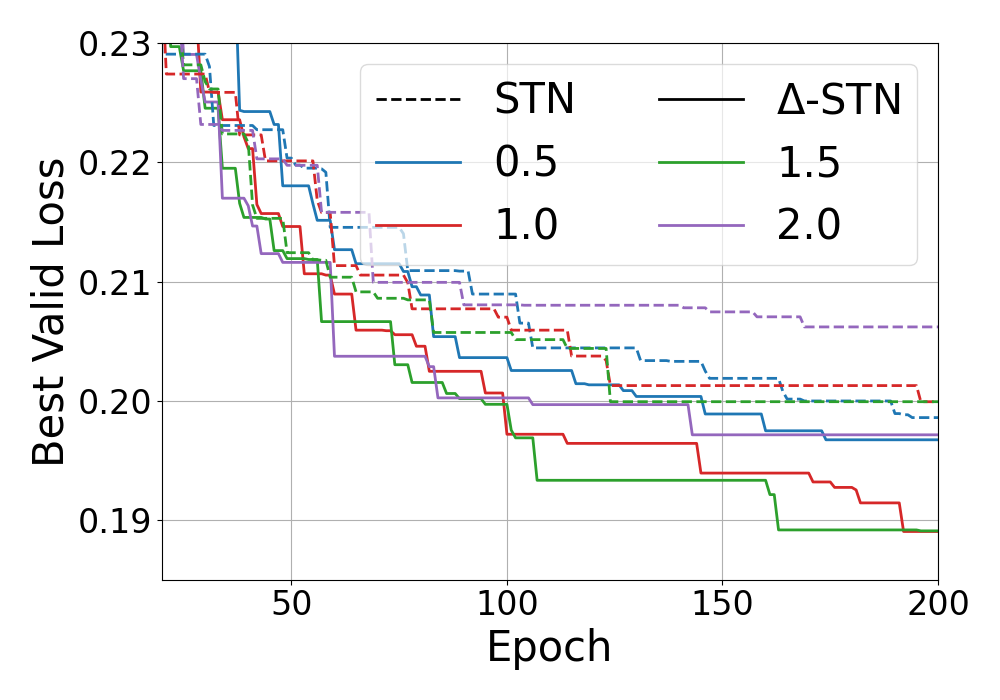}\label{fig:alexnet-scale}}
     \subfigure[AlexNet]{\includegraphics[width=0.45\textwidth]{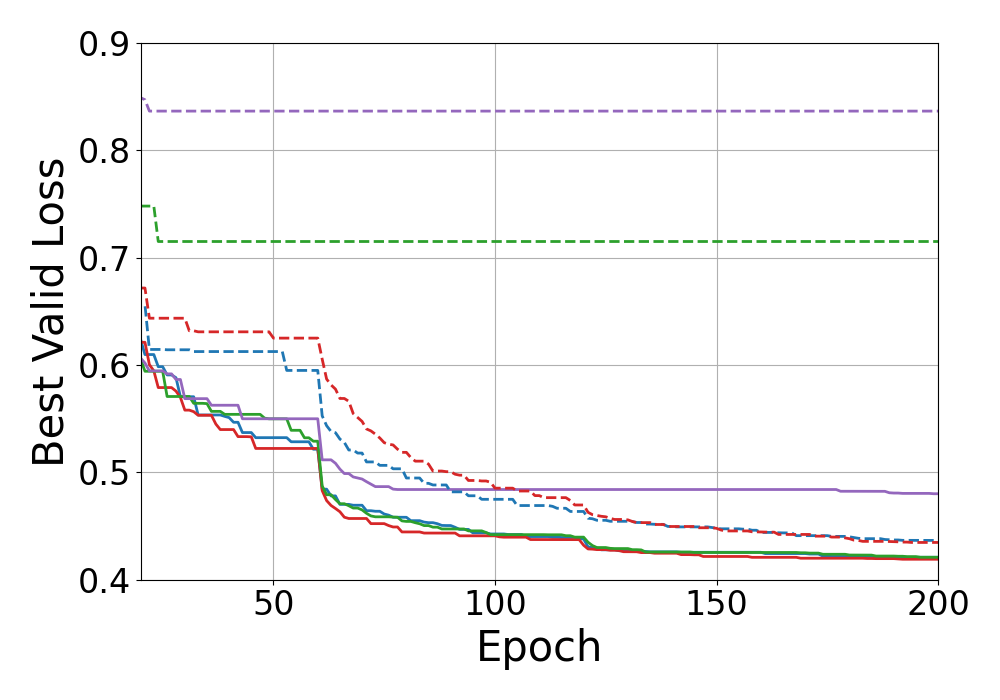}\label{fig:alexnet-steps}}
     \caption{The effect of using different perturbation scale on \textbf{(a)} SimpleCNN and \textbf{(b)} AlexNet for STNs and $\Delta$-STNs. $\Delta$-STN is more robust to a wider range of perturbation scale.}
     \label{fig:sensiv2}
\end{figure}

We show the sensitivity of $\Delta$-STNs to meta-parameters. Specifically, we investigated the effect of using different training and validation update intervals (figure~\ref{fig:sensiv}), and different fixed perturbation perturbation scale (figure \ref{fig:sensiv2}) on SimpleCNN and AlexNet. $\Delta$-STNs showed more robustness to different perturbation scale.

\begin{figure}
    \centering
    \includegraphics[width=6cm]{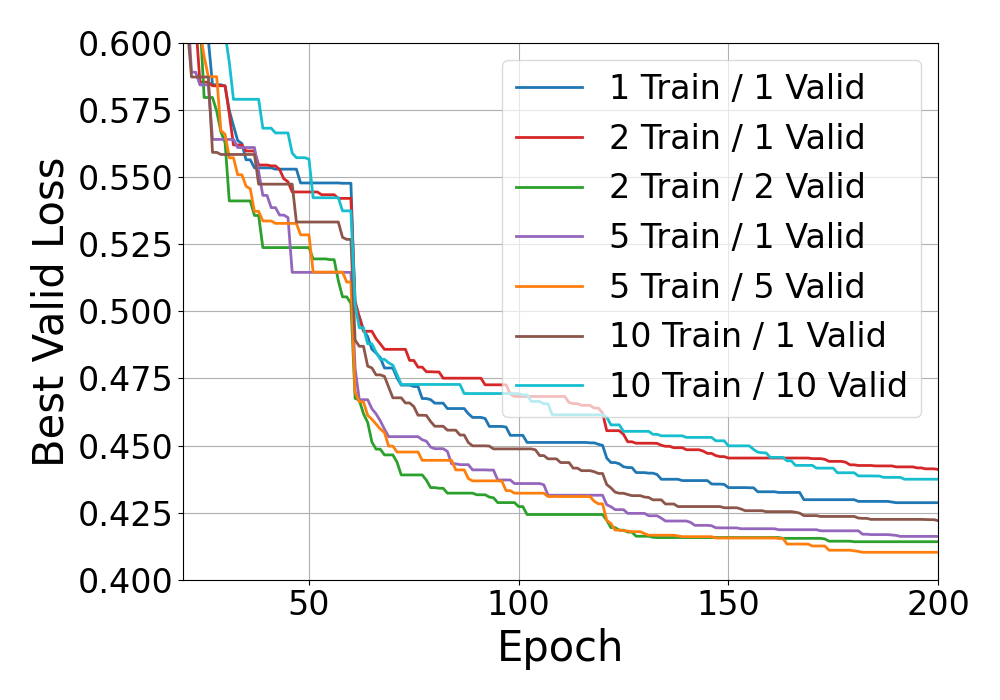}
    \caption{The effect of using different training and validation update intervals on $\Delta$-STNs}
    \label{fig:sensiv}
\end{figure}

\end{document}